\documentclass[11pt]{article} 
\usepackage[sectionbib]{natbib}
\usepackage{array,epsfig,fancyhdr,rotating}
\usepackage[]{hyperref}
\usepackage{sectsty, secdot}
\sectionfont{\fontsize{12}{14pt plus.8pt minus .6pt}\selectfont}
\renewcommand{\theequation}{\thesection\arabic{equation}}
\subsectionfont{\fontsize{11}{14pt plus.8pt minus .6pt}\selectfont}
\usepackage[margin=1.2in]{geometry}
\usepackage{setspace}
\usepackage{amsmath}
\usepackage{amssymb}
\usepackage{amsfonts}
\usepackage{multirow}
\usepackage{amsthm}
\usepackage{bm}
\usepackage{algorithm}
\usepackage{algorithmicx}
\usepackage{algpseudocode}
\usepackage{caption}
\usepackage{booktabs}
\usepackage{multirow}
\usepackage{xcolor}
\usepackage[normalem]{ulem}
\usepackage{comment}
\usepackage{subfigure}
\usepackage{tabularx}
\usepackage{array}

\newtheorem{theorem}{Theorem}

\newtheorem{cor}{Corollary}

\theoremstyle{definition}
\newtheorem{definition}{Definition}
\newtheorem{as}{Assumption}

\newtheorem{remark}{Remark}

\begin{document}
\setlength{\emergencystretch}{4em}\sloppy

\fontsize{11}{14pt plus.8pt minus.6pt}\selectfont \vspace{0.8pc}
\centerline{\large\bf A convolutional framework for detecting}
\vspace{2pt}
\centerline{\large\bf event-driven dynamics in energy price series}
\vspace{.4cm}
\centerline{Caixia Xu\textsuperscript{1,2,*}, Piotr Fryzlewicz\textsuperscript{2}}
\vspace{.4cm}
\centerline{\textsuperscript{1}\textit{Shanghai University of Finance and Economics}}
\centerline{\textsuperscript{2}\textit{London School of Economics}}
\vspace{.55cm} \fontsize{9}{11.5pt plus.8pt minus.6pt}\selectfont
\begingroup
\renewcommand{\thefootnote}{\fnsymbol{footnote}}
\footnotetext[1]{Corresponding author. Email: C.Xu44@lse.ac.uk}
\endgroup

\begin{quotation}
\noindent {\it Abstract:}\\
This paper develops a general convolutional neural network (CNN) framework for detecting heterogeneous event-driven dynamics in univariate time series windows. 
We show that the induced CNN class exactly represents classifiers based on range, maximum drawup, maximum drawdown and slope change, and uniformly approximates realised volatility and autoregressive explosiveness on compact domains. We further establish error bounds for representative rules in finite samples and an oracle inequality for learning across them. Simulations show that the proposed model can match or outperform classifiers based on individual statistics as the training sample grows. In an application to six daily energy price series, a hierarchical CNN distinguishes event windows and event families. Applied without retraining to observations withheld after 20 February 2026, the fitted model identifies predominantly geopolitical dynamics in several oil and refined product series around the outbreak of the 2026 Iran war, while distinguishing a contemporaneous natural gas spike associated with weather.

\vspace{9pt}

\vspace{3pt}

\par
\end{quotation}\par

\def\thefigure{\arabic{figure}}
\def\thetable{\arabic{table}}

\renewcommand{\theequation}{\thesection.\arabic{equation}}

\fontsize{11}{14pt plus.8pt minus .6pt}\selectfont

\section{Introduction}





Detecting abnormal dynamics in price time series is a longstanding problem in econometrics and financial statistics. Such dynamics may appear as an unusually large price range, rapid price increases, sharp declines, volatility bursts, changes in trend, or persistent departures from an earlier regime. Depending on the economic context, these patterns may be interpreted as speculative bubbles, structural changes, episodes of market stress, or price responses to external events. Their identification is important for market surveillance, risk management and the study of price formation.

Classical probability theory characterises the range of a random-walk path, with \citet{feller1951asymptotic} deriving its asymptotic distribution through a Brownian motion approximation. In financial applications, \citet{parkinson1980extreme} used the squared logarithmic range within a period to estimate return variance. Maximum drawdown and drawup record the largest movement from a peak to a later trough and from a trough to a later peak \citep{magdonismail2004maximum, hadjiliadis2006drawdowns}. Realised volatility aggregates squared returns sampled at high frequency \citep{andersen2003modeling}. One influential strand of the econometric literature formalises abnormal price dynamics as episodes of autoregressive explosiveness and develops recursive procedures using right-tailed unit root tests for detection and dating \citep{phillips2011dating, phillips2015testing}. Recent extensions consider unit root tests based on quantile regression and monitoring procedures \citep{wu2025quantile}, while alternative approaches develop stochastic models of bubble formation and asset pricing \citep{gourieroux2025stochastic}.

More generally, a broader statistical literature studies abnormal dynamics through change-point methods designed for particular forms of departure, including changes in mean, variance, trend, and other model parameters \citep{fryzlewicz2014wild, gao2019variance, fearnhead2019detecting}. Although general change-point frameworks can accommodate several forms of abnormal dynamics \citep{killick2012optimal, baranowski2019narrowest}, their application remains conditional on a previously specified statistical characterisation of the change of interest. These procedures often provide sharp statistical guarantees under their respective assumptions. A statistic designed to detect one form of abnormal behaviour, such as autoregressive explosiveness, level shifts, variance breaks or changes in slope, may not perform equally well for others. We therefore ask whether there is a method that can represent or approximate all these statistics within a single general and simple architecture.

This question naturally motivates the use of neural networks, which a growing literature employs to reduce reliance on a single manually constructed detector. In applications to asset bubbles, \citet{bashchenko2020deep} train a bidirectional LSTM on simulated price paths, and \citet{biagini2025detecting} train a feedforward neural network on call option prices. Both approaches define a bubble through the strict local martingale property of the discounted asset price. In change-point analysis, recurrent representations have been used to learn kernels for a maximum mean discrepancy (MMD) statistic for two samples \citep{chang2019kernel}. For sequential detection, \citet{lee2023training} train a classifier whose logit approximates a log likelihood ratio and then accumulate this quantity through a CUSUM recursion. These studies demonstrate the flexibility of neural architectures, but their detection targets remain tied to a particular stochastic definition, discrepancy measure or sequential testing construction. A closer theoretical bridge is provided by \citet{li2024automatic}, who show that CUSUM and generalised CUSUM classifiers can be represented by ReLU networks and establish finite-sample learning guarantees.
However, a price series may exhibit several distinct forms of change, and a single CUSUM statistic need not capture all of them. This motivates a lean architecture capable of accommodating a broad collection of statistical detectors within a common model.

Our contributions are as follows. First, building on the representation perspective of \citet{li2024automatic}, we develop a common convolutional framework that accommodates different statistics for detecting abnormal dynamics. Second, we establish that the induced CNN class contains classifiers based on the range, maximum drawup, maximum drawdown and slope change statistics and uniformly approximates classifiers based on realised volatility and autoregressive explosiveness on compact domains. We complement these results with bounds for representative rules in finite samples and an oracle inequality comparing learning in the common class with fixed statistical comparators. Third, we provide a hierarchical empirical implementation that first detects event windows and then distinguishes weather, geopolitical and supply-financial events in six energy price series. An independent holdout study then applies the fitted hierarchy without retraining to the onset of the 2026 Iran war, illustrating how its event-family classifications can be interpreted across heterogeneous energy markets.

The remainder of the paper is organised as follows. Section~\ref{method} introduces the statistical detection problem and the common CNN architecture. Section~\ref{Asymptotic} establishes the representation, approximation and learning results for finite samples. Section~4 presents the oracle comparison simulation. Section~\ref{Application} describes the energy price application, including the binary and multiclass analyses and the independent case study. Section~\ref{Conclusion} concludes and discusses directions for future research.

\section{Methodology}\label{method}

\subsection{\texorpdfstring{Classifiers based on classical statistics}{Classifiers based on classical statistics}}\label{sec:setting}

Consider a window $\bm{X}=(X_1,\cdots,X_T) \in \mathbb{R}^T$ drawn from a univariate energy commodity price series. The task is to decide whether the window contains an event-driven price pattern. The positive class corresponds to windows overlapping with events of interest, while the negative class corresponds to non-event windows. Let $Y\in\{0,1\}$ denote the class label and let $\pi_y=\mathbb P(Y=y)>0$, $y\in\{0,1\}$.

Because event windows may display different forms of abnormal behaviour, we consider several classical statistics designed to capture distinct features of a price path. Figure~\ref{fig:classical_statistics_examples} gives observed examples of the price patterns captured by these statistics before their formal definitions below.


\begin{figure}[H]
    \centering
    \includegraphics[width=\linewidth]{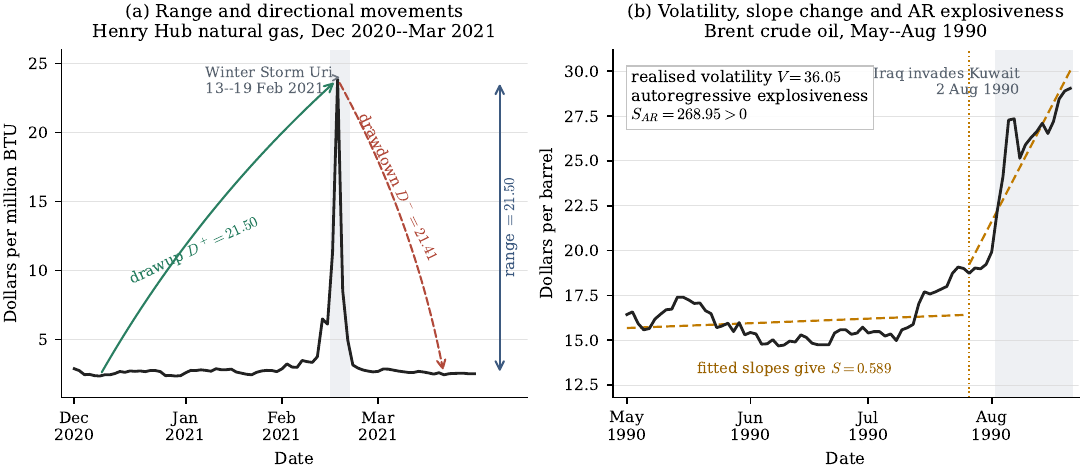}
    \caption{Observed event-driven price patterns and the corresponding classical statistics. The shaded intervals mark Winter Storm Uri in Panel~(a) and Iraq's invasion of Kuwait in Panel~(b). All reported statistics are computed over the complete 80-observation window shown in each panel.}
    \label{fig:classical_statistics_examples}
\end{figure}

\subsubsection{Range classifier}
For $T\geq1$, we first consider the range statistic \citep{feller1951asymptotic} $\mathcal{R}(\bm{X})= \max_{1 \le t \le T}X_t - \min_{1 \le t \le T}X_t$, which measures the overall amplitude of price movement within the window. A large range indicates that the price has moved substantially during the window, regardless of the direction or timing of the movement. The corresponding range classifier is defined by
\begin{equation}
    h_{R,\lambda}=\mathbb{I}\{ \mathcal{R}(\bm{X})> \lambda\}.
    \label{eq:range}
\end{equation}

\subsubsection{Directional classifier}
The second group of statistics consists of directional cumulative movement statistics \citep{magdonismail2004maximum, hadjiliadis2006drawdowns}, including maximum drawup and drawdown. For $T\geq2$, the maximum drawup statistic is defined by $D^+(\bm{X})=\max_{1 \le s < t \le T} (X_t-X_s)$, and the maximum drawdown statistic is defined by $D^-(\bm{X})=\max_{1 \le s < t \le T} (X_s-X_t)$. The maximum drawup captures the strongest cumulative upward movement within the window, while the maximum drawdown captures the strongest cumulative downward movement. These two statistics are more directional than the range statistic and are relevant for distinguishing upward price pressure from downward market stress. The corresponding classifiers are
\begin{equation}
    h_{D^+,\lambda}(\bm{X})=\mathbb{I}\{D^+(\bm{X})>\lambda\}, \quad
    h_{D^-,\lambda}(\bm{X})=\mathbb{I}\{D^-(\bm{X})>\lambda\}.
    \label{eq:drawup/down}
\end{equation}

\subsubsection{Realised volatility classifier}
Another statistic we consider is a realised volatility statistic \citep{andersen2003modeling}. For $T\geq2$, define the single-period return as $r_t=X_{t+1}-X_t$ for $t=1,\cdots,T-1$. The realised volatility statistic is $V(\bm{X})=\sum_{t=1}^{T-1}r_t^2$. The statistic captures the cumulative magnitude of local price fluctuations and is useful for detecting windows with unusually high market instability. The corresponding classifier can be defined as
\begin{equation}
    h_{V,\lambda}=\mathbb{I}\{V(\bm{X})>\lambda\}.
    \label{eq:volatility}
\end{equation}

\subsubsection{Slope change classifier}
To capture changes in local trend, we consider a slope change statistic \citep{fearnhead2019detecting}. For $T\geq5$, let \(\mathcal I=\{3,\ldots,T-2\}\) be the set of candidate split points. For \(\tau\in\mathcal I\), define $\bar{t}_L(\tau)=(1+\tau)/2$, $\bar{t}_R(\tau)=(\tau+1+T)/2$, $D_L(\tau)=\sum_{t=1}^{\tau}(t-\bar{t}_L(\tau))^2$, $D_R(\tau)=\sum_{t=\tau+1}^{T}(t-\bar{t}_R(\tau))^2$, then the least squares slopes fitted on the left and right parts of the window can be defined by
\begin{equation*}
    \hat{b}_L(\tau)=\frac{\sum_{t=1}^{\tau}(t-\bar{t}_L(\tau))X_t}{D_L(\tau)},\quad \hat{b}_R(\tau)=\frac{\sum_{t=\tau+1}^{T}(t-\bar{t}_R(\tau))X_t}{D_R(\tau)}.
\end{equation*}
Define $C_{\tau}(\bm{X})=\hat{b}_L(\tau)-\hat{b}_R(\tau)$, the slope change statistic is $S(\bm{X})=\max_{\tau\in\mathcal I}|C_{\tau}(\bm{X})|$. The slope change classifier can be defined by
\begin{equation}
    h_{S,\lambda}=\mathbb{I}\{S(\bm{X})>\lambda\}.
    \label{eq:slope}
\end{equation}
For a threshold $\lambda_S$ fixed before observing the labelled sample, write $\overline h_S=h_{S,\lambda_S}$ for the corresponding slope comparator.

\subsubsection{Autoregressive explosiveness classifier}
Finally, for $T\geq2$, we consider an autoregressive explosiveness statistic motivated by the AR(1) model. In bubble detection and explosive behaviour testing, a common starting point is the autoregressive model $X_t=\Phi X_{t-1}+\varepsilon_t$, \(t=1,\ldots,T\), with \(X_0=0\). If $\Phi>1$, then the process exhibits locally explosive behaviour, which is often used as a statistical signature of dynamics resembling bubbles.
A natural estimator of $\Phi$ is the least squares coefficient ${\sum_{t=1}^{T-1}X_{t+1}X_{t}}/{\sum_{t=1}^{T-1}X_t^2}$.
Testing whether the least squares estimate satisfies $\hat{\Phi}>1$ is equivalent to checking whether $\sum_{t=1}^{T-1}X_t(X_{t+1}-X_t)>0$ provided that $\sum_{t=1}^{T-1}X_t^2>0$. Therefore, we consider the AR statistic
\begin{equation}
    S_{AR}(\bm X)=\sum_{t=1}^{T-1}X_t(X_{t+1}-X_t).
    \label{eq:ar}
\end{equation}
For a threshold $\lambda$, define $h_{AR,\lambda}(\bm X)=\mathbb I\{S_{AR}(\bm X)>\lambda\}$. This statistic measures whether price increments tend to align with the lagged price level. It is a simple and stylised statistic for explosive behaviour and is closely related to classical bubble detection ideas based on autoregressive coefficients \citep{phillips2011dating, phillips2015testing}.

These baseline statistics cover several interpretable forms of price behaviour associated with events, including large movements, directional cumulative changes, volatility bursts, trend changes and autoregressive explosiveness. The next subsection introduces a common CNN architecture for representing or approximating these statistics. 



\subsection{1D CNN event classifier}\label{sec:cnn}

A natural motivation for the use of neural networks is provided by their ability to represent classical statistical decision rules. Our goal is to provide a simple architecture that can represent or approximate all these classifiers within a single neural network.

The classifier defined below is a univariate convolutional neural network (1D CNN) designed for binary classification. Informally, a convolutional filter moves along the price window and applies the same local transformation at each time position. It is closely related to a moving sum (MOSUM) statistic at the computational level, as both apply a common local weight pattern while scanning along a time series \citep{eichinger2018mosum}. Nevertheless, a MOSUM procedure requires the parameter of interest and its local estimator to be specified in advance, while the proposed CNN provides a flexible framework for all the heterogeneous window classification problems considered here.


As shown in Figure \ref{fig:CNN}, for each commodity, the input is a price window $\bm{X}=(X_{1},\ldots,X_{T})$ of length $T$. The network consists of $B$ parallel convolutional branches with different kernel sizes, depths, channel sizes, and pooling operations. For each branch $b=1,\cdots,B$, let $L_b$ be the number of convolutional blocks in branch $b$. Let $C_{b,0},C_{b,1},\cdots,C_{b,L_b}$ be the channel sizes, where $C_{b,0}=1$ because the input has one channel. Let $k_{b,1},\cdots,k_{b,L_b}$ be the convolutional kernel sizes. We take the convolutional stride equal to one without padding. The parallel branches provide multiple filters with different temporal spans, allowing the network to examine price movements at different temporal scales.

 \begin{figure}[h]
     \centering
     \includegraphics[width=\linewidth, trim=6cm 0cm 0cm 0cm, clip]{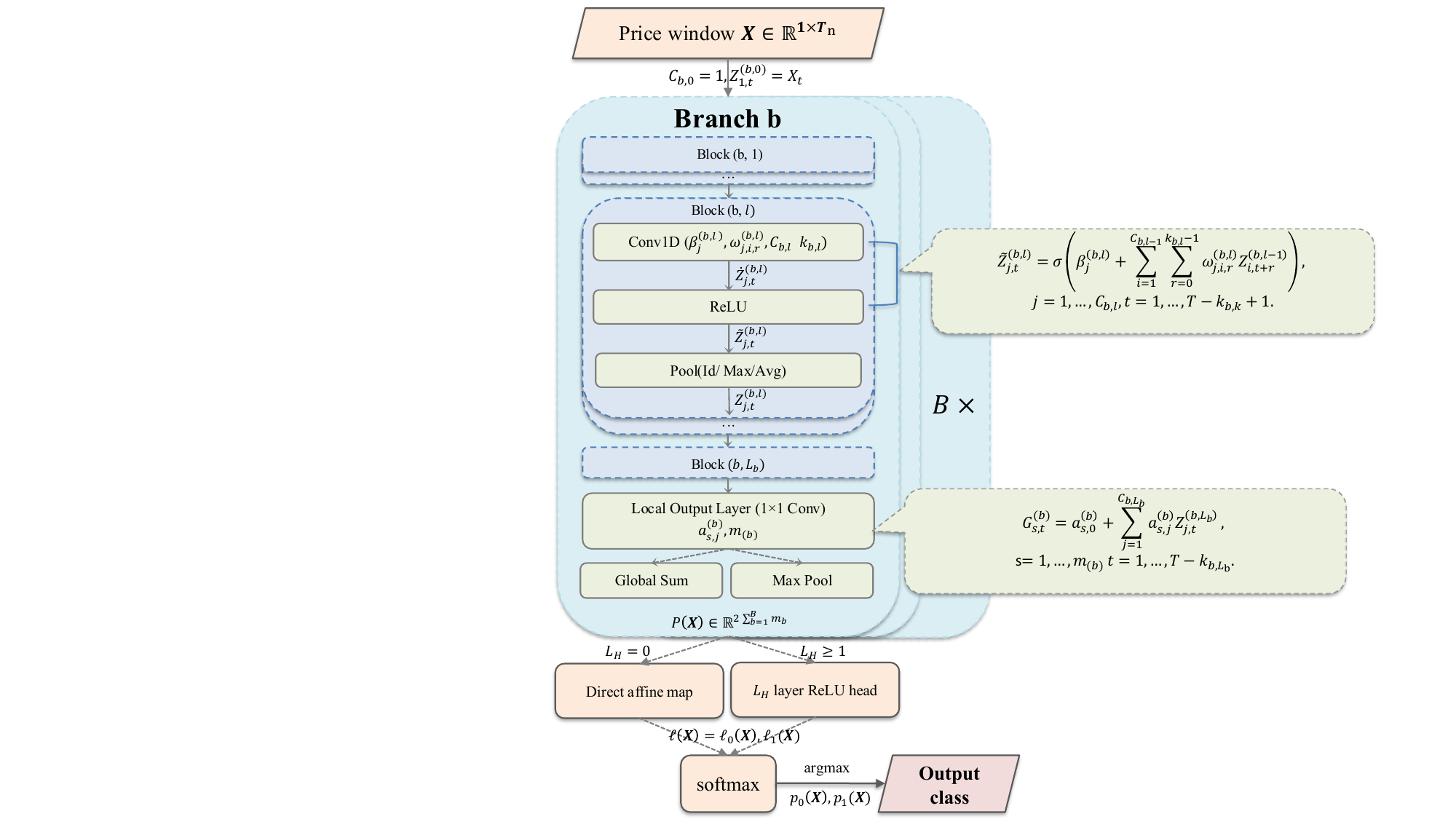}
     \caption{1D convolutional neural network for rolling window binary classification}
     \label{fig:CNN}
 \end{figure}

For each branch, define the initial feature map by $Z^{(b,0)}_{(1,t)}=X_t$, $t=1,\cdots,T$. Let $\sigma(x)=x_+=\max\{x,0\}$ denote the ReLU activation function. The $l$th convolutional layer computes
\begin{equation}   \tilde{Z}_{j,t}^{(b,l)}=\sigma \big(\beta_j^{(b,l)}+\sum_{i=1}^{C_{b,l-1}}\sum_{r=0}^{k_{b,l}-1}w_{j,i,r}^{(b,l)}Z_{i,t+r}^{(b,l-1)} \big),
\end{equation}
for $j=1,\cdots,C_{b,l}$. Here $w_{j,i,r}^{(b,l)}$ is the convolutional filter weight from input channel $i$ to output channel $j$ at offset $r$, and $\beta_j^{(b,l)}$ is the bias of output channel $j$. Using the same filter weights at every admissible position allows a local pattern to be recognised throughout the window without assigning separate parameters to individual time positions. Stacking convolutional layers then combines these local responses into features covering progressively longer time intervals. Each convolutional layer is followed by a pooling operation
$\mathrm{Pool}_{b,l}\in\{\mathrm{Id},\;\mathrm{MaxPool}_{w_{b,l},s_{b,l}},\;
\mathrm{AvgPool}_{w_{b,l},s_{b,l}}\}$ applied separately to each channel along the time axis, with window $w_{b,l}\ge 1$ and stride $s_{b,l}\ge 1$. Set \(T_{b,0}=T\). Before pooling in block \(l\), the convolutional output has length $\widetilde T_{b,l}
= T_{b,l-1}-k_{b,l}+1$, which satisfies $1\leq k_{b,l}\leq T_{b,l-1}$. Thus the convolutional time index is $t=1,\ldots,\widetilde T_{b,l}.$ If $\operatorname{Pool}_{b,l}=\operatorname{Id}$, $T_{b,l}=\widetilde T_{b,l}$. Otherwise, we require $1\leq w_{b,l}\leq\widetilde T_{b,l}$ and set $T_{b,l}
=\left\lfloor (\widetilde T_{b,l}-w_{b,l})/{s_{b,l}}
\right\rfloor+1 $.


We then define $m_b$ local output channels by a linear layer with kernel size one
\begin{equation}
G_{s,t}^{(b)}=a_{s,0}^{(b)}+\sum_{j=1}^{C_{b,L_b}}a_{s,j}^{(b)}Z_{j,t}^{(b,L_b)}
\end{equation}
for $s=1,\cdots,m_b$, $t=1,\cdots,T_{b,L_b}$. This pointwise layer recombines the final convolutional channels into candidate local scores without altering their time positions. The local output channels are then aggregated into features for the entire window by global pooling. For each branch $b$ and channel $s$, define
\begin{equation}
    P_s^{(b),\max}=\max_{1\le t \le T_{b,L_b}}G_{s,t}^{(b)},\quad P_s^{(b),\text{sum}}=\sum_{t=1}^{T_{b,L_b}}G_{s,t}^{(b)}.
\end{equation}
The feature obtained by maximum pooling captures the strongest local response in the window, while the feature obtained by sum pooling captures its cumulative strength. Together, these operations convert local scores into evidence for the complete window, remove the temporal dimension and make the resulting features less sensitive to the position of a response.

Collect all pooled features into \(\bm P(\bm X)\in\mathbb R^{d_0}\), where \(d_0=2\sum_{b=1}^B m_b\), and fix a head depth \(L_H\geq0\). The head architecture is specified before model fitting. When \(L_H\geq1\), fix hidden widths \(\bm d_H=(d_1,\ldots,d_{L_H})\), each of which is at least two, set \(\bm z^{(0)}(\bm X)=\bm P(\bm X)\), and define $\bm z^{(l)}(\bm X)=\sigma\!\left(A_l\bm z^{(l-1)}(\bm X)+\bm b_l\right), l=1,\ldots,L_H.$
The two output scores are $\bm\ell(\bm X)=A_{L_H+1}\bm z^{(L_H)}(\bm X)+\bm b_{L_H+1}\in\mathbb R^2.$ The width condition retains every affine score \(g(\bm P)=\gamma_0+\bm\gamma^\top\bm P\) as a submodel because two hidden units can propagate \(\sigma(g(\bm P))\) and \(\sigma(-g(\bm P))\), whose output difference equals \(g(\bm P)\).
When \(L_H=0\), the scores are defined directly by \(\bm\ell(\bm X)=A_1\bm P(\bm X)+\bm b_1\). When fitted, the head learns from labelled windows how to select, reweight and combine the pooled evidence. The components $\ell_0(\bm X)$ and $\ell_1(\bm X)$ correspond to the classes without and with events. They are converted into class probabilities by the softmax transformation
\[p_c(\bm X)
=
\frac{\exp\{\ell_c(\bm X)\}}{\exp\{\ell_0(\bm X)\}+\exp\{\ell_1(\bm X)\}},
\qquad
c\in\{0,1\},\]
and the predicted class is $\widehat y(\bm X)=\arg\max_{c\in\{0,1\}}\ell_c(\bm X)=\arg\max_{c\in\{0,1\}}p_c(\bm X)$. In either case, define \(f(\bm X)=\ell_1(\bm X)-\ell_0(\bm X)\) and $h(\bm X)=\mathbb I\{f(\bm X)>0\}$. 

The resulting network and its associated classifier class are defined as follows.

\begin{definition}
Let $\mathcal{F}_{\text{CNN}}(T,B,\bm{L},\bm{k},\bm{\mathcal{P}},\bm{C},\bm{m};L_H,\bm d_H)$ be the class of functions $f$ constructed above, where $T$ is the input length, $B$ is the number of branches, $\bm{L}=(L_1,\cdots,L_B)$ is the number of blocks in each branch, $\bm{k}=\{k_{b,l}: 1 \le b \le B, 1 \le l \le L_b\}$ gives the kernel sizes, $\bm{\mathcal{P}}=\{(\mathrm{Pool}_{b,l},w_{b,l},s_{b,l}): 1 \le b \le B, 1 \le l \le L_b \}$ gives the pooling operators, $\bm{C}=\{C_{b,l}: 1 \le b \le B, 0 \le l \le L_b \}$ gives the channel sizes with $C_{b,0}=1$, and $\bm{m}=(m_1,\cdots,m_B)$ gives the numbers of local output channels. The head is specified by \(L_H\) and \(\bm d_H\), and all weights and biases are free parameters. The associated classifier class is
\begin{equation}
    \mathcal{H}_{\text{CNN}}(T,B,\bm{L},\bm{k},\bm{\mathcal{P}},\bm{C},\bm{m};L_H,\bm d_H)=\{\mathbb{I}\{f>0\}: f \in \mathcal{F}_{\text{CNN}}(T,B,\bm{L},\bm{k},\bm{\mathcal{P}},\bm{C},\bm{m};L_H,\bm d_H)\}.
\end{equation}
\(L_H\) and \(\bm d_H\) are understood to be fixed when the head specification is omitted in the remainder of the paper. 
\label{def:cnn}
\end{definition}


\section{Representation, approximation and risk bounds}\label{Asymptotic}

In Section \ref{method}, we introduced the classical statistics considered in the paper and proposed the CNN architecture. It remains to show that the CNN classifier class in Definition~\ref{def:cnn} is rich enough to contain or uniformly approximate the classifiers based on classical statistics introduced in Section \ref{sec:setting}. We also derive finite-sample error bounds for classifiers based on slope change, realised volatility and autoregressive explosiveness, and a generalization bound for empirical risk minimisation over the common CNN class relative to these rules. Proofs are given in the Supplementary Material.

\subsection{Exact representation for range, directional and slope classifiers}

We begin with the range classifier, which detects unusually large movements.

\begin{theorem}[Exact CNN representation of the range score]
Let \(T\geq1\), $\bm L=(1), \bm k=\{k_{1,1}=1\}, \bm{\mathcal P}=\{(\operatorname{Id},1,1)\}, \bm C=\{C_{1,0}=1,C_{1,1}=2\}, \bm m=(2).$ For every \(\lambda>0\), there exists $f_{R,\lambda}\in \mathcal F_{\mathrm{CNN}}
(T,1,\bm L,\bm k,\bm{\mathcal P},\bm C,\bm m)$ such that $f_{R,\lambda}(\bm X)=\mathcal R(\bm X)-\lambda$ for every $\bm X\in\mathbb R^T$. Consequently, $h_{R,\lambda}\in
\mathcal H_{\mathrm{CNN}}
(T,1,\bm L,\bm k,\bm{\mathcal P},\bm C,\bm m).$
\label{thm:range}
\end{theorem}

Theorem~\ref{thm:range} holds on all of \(\mathbb R^T\) and requires no bound on the input. The construction produces \(\sigma(X_t)\) and \(\sigma(-X_t)\) in two channels and recovers \(X_t\) and \(-X_t\) using \(x=\sigma(x)-\sigma(-x)\). 

We next consider the directional classifiers.

\begin{theorem}[Exact CNN representation of the directional classifiers]
Let \(T\geq2\), \(B=T-1\), $\bm L=(1,\ldots,1)$, $\bm k=\{k_{h,1}=h+1:h=1,\ldots,T-1\}$, $\bm{\mathcal P}=
\{(\operatorname{Pool}_{h,1},w_{h,1},s_{h,1})
=(\operatorname{Id},1,1):h=1,\ldots,T-1\}$, $\bm C=
\{C_{h,0}=1,C_{h,1}=1:h=1,\ldots,T-1\}$, $\bm m=(1,\ldots,1)$. For every \(\lambda>0\), there exist $f_\lambda^+,f_\lambda^-
\in\mathcal F_{\mathrm{CNN}}
(T,B,\bm L,\bm k,\bm{\mathcal P},\bm C,\bm m)$ such that, for every \(\bm X\in\mathbb R^T\), $\mathbb I\{f_\lambda^+(\bm X)>0\}
=\mathbb I\{D^+(\bm X)>\lambda\}$, $\mathbb I\{f_\lambda^-(\bm X)>0\} =\mathbb I\{D^-(\bm X)>\lambda\}$. Consequently, $h_{D^+,\lambda},h_{D^-,\lambda}
\in\mathcal H_{\mathrm{CNN}}
(T,B,\bm L,\bm k,\bm{\mathcal P},\bm C,\bm m)$.
\label{thm:drawup/down}
\end{theorem}


For every candidate split point, the slope contrast is a linear function of the whole input window. We therefore use one branch spanning the full window for each split point and each sign, and the affine head combines the resulting nonnegative branch scores. 

\begin{theorem}[Exact CNN representation of the slope change classifier]
Let \(T\geq5\), \(\mathcal I=\{3,\ldots,T-2\}\), and \(B=2(T-4)\). Let $\bm L=(1,\ldots,1)$, $\bm k=\{k_{b,1}=T:b=1,\ldots,B\}$, $\bm{\mathcal P}=
\{(\operatorname{Pool}_{b,1},w_{b,1},s_{b,1})
=(\operatorname{Id},1,1):b=1,\ldots,B\}$, $\bm C=\{C_{b,0}=1,C_{b,1}=1:b=1,\ldots,B\}$, $\bm m=(1,\ldots,1)$. For every \(\lambda>0\), there exists $f_{S,\lambda}\in
\mathcal F_{\mathrm{CNN}}
(T,B,\bm L,\bm k,\bm{\mathcal P},\bm C,\bm m)$ such that $\mathbb I\{f_{S,\lambda}(\bm X)>0\}
=\mathbb I\{S(\bm X)>\lambda\}$ for every $\bm X\in\mathbb R^T$. Consequently, $h_{S,\lambda}\in \mathcal H_{\mathrm{CNN}} (T,B,\bm L,\bm k,\bm{\mathcal P},\bm C,\bm m)$.
\label{thm:trend}
\end{theorem}



The range construction in Theorem~\ref{thm:range} reproduces the score \(\mathcal R(\bm X)-\lambda\) exactly. In contrast, Theorems~\ref{thm:drawup/down} and \ref{thm:trend} establish exact representation at the classifier level. Their constructed scores are positive precisely when the corresponding statistics exceed \(\lambda\).

\subsection{Uniform approximation for realised volatility and autoregressive statistics}

We now turn from exact representation to uniform approximation. The realised volatility and autoregressive constructions below approximate the corresponding statistic values uniformly on compact domains.

For \(M>0\), let \(\mathcal X_M=[-M,M]^T\) and define componentwise clipping by
\[
[\operatorname{clip}_M(\bm x)]_t=(-M)\vee(x_t\wedge M),
\qquad t=1,\ldots,T.
\]
Equivalently, \(\operatorname{clip}_M(x)=-M+\sigma(x+M)-\sigma(x-M)\), so clipping is a fixed piecewise linear preprocessing map that can be incorporated using ReLU layers with kernels of size one. For an arbitrary input \(\bm X\), define \(E_M=\{\|\bm X\|_\infty\leq M\}\), and let \(\tau_y(M)=\mathbb P(E_M^c\mid Y=y)\). If \(f_{\epsilon,M}^{G}\) approximates \(G\) on \(\mathcal X_M\), it is applied to arbitrary inputs through \(\bar f_{\epsilon,M}^{G}(\bm X)=f_{\epsilon,M}^{G}(\operatorname{clip}_M(\bm X))\). Consequently, \(\lvert\bar f_{\epsilon,M}^{G}(\bm X)-G(\operatorname{clip}_M(\bm X))\rvert\leq\epsilon\) for every \(\bm X\in\mathbb R^T\).

The remaining statistics contain quadratic local terms. A finite ReLU network approximates these terms uniformly on \(\mathcal X_M\), after which the clipped composition extends the score to arbitrary inputs.
\begin{theorem}[Uniform CNN approximation of the realised volatility statistic]
Let \(T\geq2\). For any \(M>0\) and \(\epsilon\in(0,1)\), let \(\bm C_\epsilon\) be the collection of channel sizes used in the ReLU square approximant and take \(\bm L=(L_\epsilon)\), \(\bm k=\{k_{1,1}=2,k_{1,l}=1:l=2,\ldots,L_\epsilon\}\), \(\bm{\mathcal P}=\{(\operatorname{Pool}_{1,l},w_{1,l},s_{1,l})=(\operatorname{Id},1,1):l=1,\ldots,L_\epsilon\}\) and \(\bm m=(1)\), where \(L_\epsilon=O\{\log(2+(T-1)M^2/\epsilon)\}\). Then there exists \(f_{\epsilon,M}^{RV}\in\mathcal F_{\mathrm{CNN}}(T,1,\bm L,\bm k,\bm{\mathcal P},\bm C_\epsilon,\bm m)\) such that \(\sup_{\bm X\in\mathcal X_M}\lvert f_{\epsilon,M}^{RV}(\bm X)-V(\bm X)\rvert\leq\epsilon\). For every \(\lambda\in\mathbb R\),
$
\mathbb I\{f_{\epsilon,M}^{RV}(\bm X)>\lambda\}
=\mathbb I\{V(\bm X)>\lambda\},
\bm X\in\mathcal X_M,
\lvert V(\bm X)-\lambda\rvert>\epsilon.
$
\label{thm:volatility}
\end{theorem}

Theorem~\ref{thm:volatility} turns uniform score approximation into classifier agreement on \(\mathcal X_M\). Outside the neighbourhood of radius \(\epsilon\) around \(\lambda\), the CNN and realised volatility rule make the same decision on that domain.

The same construction with local convolution and global summation applies to the AR statistic, using a ReLU network that approximates multiplication.
\begin{theorem}[Uniform CNN approximation of the autoregressive statistic]
Let \(T\geq2\). For any \(M>0\) and \(\epsilon\in(0,1)\), let \(\bm C_\epsilon\) be the collection of channel sizes used in the ReLU multiplication approximant and take \(\bm L=(L_\epsilon)\), \(\bm k=\{k_{1,1}=2,k_{1,l}=1:l=2,\ldots,L_\epsilon\}\), \(\bm{\mathcal P}=\{(\operatorname{Pool}_{1,l},w_{1,l},s_{1,l})=(\operatorname{Id},1,1):l=1,\ldots,L_\epsilon\}\) and \(\bm m=(1)\), where \(L_\epsilon=O\{\log(2+(T-1)M^2/\epsilon)\}\). Then there exists \(f_{\epsilon,M}^{AR}\in\mathcal F_{\mathrm{CNN}}(T,1,\bm L,\bm k,\bm{\mathcal P},\bm C_\epsilon,\bm m)\) such that \(\sup_{\bm X\in\mathcal X_M}\lvert f_{\epsilon,M}^{AR}(\bm X)-S_{AR}(\bm X)\rvert\leq\epsilon\). For every \(\lambda\in\mathbb R\),
$
\mathbb I\{f_{\epsilon,M}^{AR}(\bm X)>\lambda\}
=\mathbb I\{S_{AR}(\bm X)>\lambda\},
\bm X\in\mathcal X_M,
\lvert S_{AR}(\bm X)-\lambda\rvert>\epsilon.
$
\label{thm:ar-approximate}
\end{theorem}

Theorem~\ref{thm:ar-approximate} gives the analogous score and decision guarantees for autoregressive explosiveness on \(\mathcal X_M\). Here the network approximates local products. Because clipping is the identity on \(E_M\), agreement with the rules based on \(V(\bm X)\) and \(S_{AR}(\bm X)\) is invoked there, while observations in \(E_M^c\) enter the later risk bounds through \(\tau_y(M)\).

\subsection{Risk bounds for classifiers based on classical statistics}

This subsection considers three classifiers based respectively on slope change, realised volatility and autoregressive explosiveness. For any classifier \(h\), write \(\operatorname{err}(h)=\mathbb P\{h(\bm X)\neq Y\}\), and let \(\pi_y=\mathbb P(Y=y)\), \(y\in\{0,1\}\). We study these rules under their respective data generating assumptions. 

\begin{as}
Conditional on \(Y=0\), suppose that \(X_t=a+bt+\varepsilon_t\) for \(t=1,\ldots,T\). Conditional on \(Y=1\), suppose that there exists \(\tau^\star\in\mathcal I\) such that \(X_t=a_L+b_Lt+\varepsilon_t\) for \(t\leq\tau^\star\) and \(X_t=a_R+b_Rt+\varepsilon_t\) for \(t>\tau^\star\), where \(|b_L-b_R|\geq\kappa_S\) for some \(\kappa_S>0\). The innovations satisfy \(\varepsilon_1,\ldots,\varepsilon_T\overset{\mathrm{i.i.d.}}{\sim}N(0,\sigma_S^2)\), where \(\sigma_S>0\), and the segment parameters and \(\tau^\star\) are independent of the innovations.
\label{ass:slope}
\end{as}

\begin{theorem}[Error bound for the slope change classifier]
\label{thm:slope_error}
Suppose Assumption~\ref{ass:slope} holds. Let \(T\geq5\), \(\mathcal I=\{3,\ldots,T-2\}\), and \(m_T=\lvert\mathcal I\rvert=T-4\). For \(\tau\in\mathcal I\), define \(v_\tau=1/D_L(\tau)+1/D_R(\tau)\), \(v_T^{\max}=\max_{\tau\in\mathcal I}v_\tau\), and, for \(\alpha_S\in(0,1)\), \(\lambda_{S,\alpha_S}=\sigma_S\sqrt{2v_T^{\max}\log(2m_T/\alpha_S)}\). Then \(h_{S,\lambda_{S,\alpha_S}}(\bm X)=\mathbb I\{S(\bm X)>\lambda_{S,\alpha_S}\}\) satisfies \(\mathbb P\{h_{S,\lambda_{S,\alpha_S}}(\bm X)=1\mid Y=0\}\leq\alpha_S\) and
\[
\mathbb P\{h_{S,\lambda_{S,\alpha_S}}(\bm X)=0\mid Y=1\}
\leq
\exp\left\{-\frac{(\kappa_S-\lambda_{S,\alpha_S})_+^2}{2\sigma_S^2v_T^{\max}}\right\},
\]
where \(x_+=\max\{x,0\}\). Consequently,
\[
\operatorname{err}(h_{S,\lambda_{S,\alpha_S}})
\leq\pi_0\alpha_S+
\pi_1\exp\left\{-\frac{(\kappa_S-\lambda_{S,\alpha_S})_+^2}{2\sigma_S^2v_T^{\max}}\right\}
=:\mathfrak B_S(\alpha_S).
\]
\end{theorem}

Theorem~\ref{thm:slope_error} calibrates the threshold to control the conditional false positive probability despite scanning multiple candidate split points. Under a slope change of magnitude at least \(\kappa_S\), the conditional false negative probability decreases exponentially with the separation between signal and noise. Combined with the false positive control, this gives the explicit risk bound \(\mathfrak B_S(\alpha_S)\).

We next turn to the realised volatility classifier and state the conditions used to derive its risk bound at finite \(T\).

\begin{as}
Let \(T\geq2\), set \(n=T-1\), and define \(r_t=X_{t+1}-X_t\). Conditional on \(Y\), suppose that \(r_1^2,\ldots,r_n^2\) are independent and that \(\lvert r_t\rvert\leq U\) almost surely for some \(U>0\). Suppose further that \(n^{-1}\sum_{t=1}^n\mathbb E(r_t^2\mid Y=0)\leq\nu_0\), whereas \(n^{-1}\sum_{t=1}^n\mathbb E(r_t^2\mid Y=1)\geq\nu_0+\kappa_V\) for some \(\kappa_V>0\).
\label{ass:volatility}
\end{as}

\begin{theorem}[Error bound for the realised volatility classifier]
\label{thm:volatility_error}
Suppose Assumption~\ref{ass:volatility} holds. Fix \(\alpha_V\in(0,1)\), define \(a_{V,\alpha_V}=U^2\sqrt{n\log(1/\alpha_V)/2}\), and set \(\lambda_{V,\alpha_V}=n\nu_0+a_{V,\alpha_V}\). Then the exact classifier \(h_{V,\lambda_{V,\alpha_V}}(\bm X)=\mathbb I\{V(\bm X)>\lambda_{V,\alpha_V}\}\) satisfies \(\mathbb P\{h_{V,\lambda_{V,\alpha_V}}(\bm X)=1\mid Y=0\}\leq\alpha_V\) and
\[
\mathbb P\{h_{V,\lambda_{V,\alpha_V}}(\bm X)=0\mid Y=1\}
\leq
\exp\left\{-\frac{2(n\kappa_V-a_{V,\alpha_V})_+^2}{nU^4}\right\}.
\]
Consequently,
\[
\operatorname{err}(h_{V,\lambda_{V,\alpha_V}})
\leq\pi_0\alpha_V+
\pi_1\exp\left\{-\frac{2(n\kappa_V-a_{V,\alpha_V})_+^2}{nU^4}\right\}.
\]

Now fix \(M>0\) and \(\epsilon_V>0\), and suppose that \(\sup_{\bm x\in\mathcal X_M}\lvert f_{\epsilon_V,M}^{RV}(\bm x)-V(\bm x)\rvert\leq\epsilon_V\). Define the clipped neural classifier $\overline h_V(\bm X)=\mathbb I\left\{f_{\epsilon_V,M}^{RV}(\operatorname{clip}_M(\bm X))>\lambda_{V,\alpha_V}+\epsilon_V\right\}.$ Then
\[
\operatorname{err}(\overline h_V)
\leq\pi_0\{\alpha_V+\tau_0(M)\}
+\pi_1\left[\exp\left\{-\frac{2(n\kappa_V-a_{V,\alpha_V}-2\epsilon_V)_+^2}{nU^4}\right\}+\tau_1(M)\right]
=:\mathfrak B_V(\alpha_V,\epsilon_V,M).
\]
\end{theorem}

The exact part of Theorem~\ref{thm:volatility_error} gives an error guarantee for finite \(T\) for the realised volatility rule. For the clipped neural rule, the same false positive control holds on \(E_M\). The approximation reduces the effective alternative separation by \(2\epsilon_V\), while inputs outside the approximation domain contribute \(\tau_0(M)\) and \(\tau_1(M)\).

Finally, we consider the AR statistic under the following Gaussian model.
\begin{as}
Let \(X_t=\Phi X_{t-1}+\varepsilon_t\), \(t=1,\ldots,T\), with \(X_0=0\). Conditional on \((Y,\Phi)\), the innovations \(\varepsilon_1,\ldots,\varepsilon_T\) are independent \(N(0,\sigma_{AR}^2)\) variables, where \(\sigma_{AR}^2>0\) is known, and they are independent of \(\Phi\) given \(Y\). The conditional law of \(\Phi\) given \(Y=0\) is supported on \([-1,1]\), whereas its conditional law given \(Y=1\) is supported on \([1+\kappa_{AR},\infty)\) for a fixed \(\kappa_{AR}>0\).
\label{ass:gaussian_ar}
\end{as}

The next result gives size and power bounds for finite \(T\) for the autoregressive rule under the Gaussian model, together with the corresponding risk bound for its clipped neural approximation.
\begin{theorem}[Error bounds for the autoregressive classifiers]
\label{thm:artest}
Suppose Assumption~\ref{ass:gaussian_ar} holds and let \(T\geq2\). Fix \(\alpha\in(0,1)\), define \(Q_T=\sum_{t=1}^{T-1}X_t^2\), and set \(\lambda_{AR,\alpha}=2\sigma_{AR}^2T\sqrt{\log(4T/\alpha)\log(2/\alpha)}\). The exact classifier \(h_{AR,\alpha}(\bm X)=\mathbb I\{S_{AR}(\bm X)>\lambda_{AR,\alpha}\}\) satisfies \(\mathbb P\{h_{AR,\alpha}(\bm X)=1\mid Y=0\}\leq\alpha\). Suppose there exist \(q_T>0\) and \(\eta_T\in[0,1]\) such that \(\mathbb P(Q_T\geq q_T\mid Y=1)\geq1-\eta_T\) and \(q_T\geq2\lambda_{AR,\alpha}/\kappa_{AR}\). Then
\[
\mathbb P\{h_{AR,\alpha}(\bm X)=0\mid Y=1\}
\leq\eta_T+\exp\left\{-\frac{\kappa_{AR}^2q_T}{8\sigma_{AR}^2}\right\}.
\]
Consequently,
\[
\operatorname{err}(h_{AR,\alpha})
\leq
\pi_0\alpha
+\pi_1\left[
\eta_T+\exp\left\{-\frac{\kappa_{AR}^2q_T}{8\sigma_{AR}^2}\right\}
\right].
\]

Now fix \(M>0\) and \(\epsilon_{AR}>0\), and suppose that \(\sup_{\bm x\in\mathcal X_M}\lvert f_{\epsilon_{AR},M}^{AR}(\bm x)-S_{AR}(\bm x)\rvert\leq\epsilon_{AR}\). Define $\overline h_{AR}(\bm X)=\mathbb I\left\{f_{\epsilon_{AR},M}^{AR}(\operatorname{clip}_M(\bm X))>\lambda_{AR,\alpha}+\epsilon_{AR}\right\}.$ If the same \(q_T\) and \(\eta_T\) satisfy the stronger condition \(q_T\geq2(\lambda_{AR,\alpha}+2\epsilon_{AR})/\kappa_{AR}\), then
\[
\operatorname{err}(\overline h_{AR})
\leq\pi_0\{\alpha+\tau_0(M)\}
+\pi_1\left[\eta_T+\exp\left\{-\frac{\kappa_{AR}^2q_T}{8\sigma_{AR}^2}\right\}+\tau_1(M)\right]
=:\mathfrak B_{AR}(\alpha,\epsilon_{AR},M).
\]
\end{theorem}

\begin{remark}
The separation \(\kappa_{AR}>0\) in Theorem~\ref{thm:artest} is fixed. Mildly explosive sequences such as \(\Phi_T=1+cT^{-a}\) require a separate analysis because both the growth of \(Q_T\) and the condition on \(q_T\) change. See, for example, \citet{phillips2015testing}. The innovation variance is also treated as known.
\end{remark}

\subsection{Generalization bounds for the common CNN class}

This subsection compares empirical risk minimisation over the common CNN class with slope, realised volatility and autoregressive rules. It first bounds the population risk of the empirical risk minimiser relative to the best classifier in the common class and then obtains an oracle comparison with the three fixed comparators. For \(1\leq d\leq N\) and \(\delta\in(0,1)\), write $r_N(d,\delta)=\sqrt{[{8d\log(2eN/d)+8\log(4/\delta)}]/{N}}$.

\begin{theorem}[Oracle bound over the three fixed statistical comparators]
\label{thm:three_rule_oracle}
Let \(T\geq5\) and let \(\mathcal D_N=\{(\bm X^{(i)},Y^{(i)})\}_{i=1}^N\) consist of independent copies of \((\bm X,Y)\). Before observing \(\mathcal D_N\), fix fully specified deterministic comparators \(\overline h_S,\overline h_V,\overline h_{AR}\) of the forms introduced above. Their specifications include every threshold and calibration constant, the clipping radius, and the selected approximation networks. There exists a fixed CNN classifier class with parallel branches \(\mathcal H\), with the head specification fixed as in Definition~\ref{def:cnn}, such that \(\overline h_S,\overline h_V,\overline h_{AR}\in\mathcal H\). Let \(d_{\mathcal H}=\operatorname{VCdim}(\mathcal H)\), where \(1\leq d_{\mathcal H}\leq N\), and let \(\widehat h_{\mathrm{ERM}}\in\arg\min_{h\in\mathcal H}N^{-1}\sum_{i=1}^N\mathbb I\{h(\bm X^{(i)})\neq Y^{(i)}\}\). Then, for every \(\delta\in(0,1)\), with probability at least \(1-\delta\),
\[
\operatorname{err}(\widehat h_{\mathrm{ERM}})
\leq\inf_{h\in\mathcal H}\operatorname{err}(h)+2r_N(d_{\mathcal H},\delta)
\leq\min_{G\in\{S,V,AR\}}\operatorname{err}(\overline h_G)+2r_N(d_{\mathcal H},\delta).
\]
\end{theorem}

The workflow used in the proof of Theorem~\ref{thm:three_rule_oracle} can be applied to other statistics, given suitable CNN representation or approximation results and comparator error bounds. The corresponding branches can then be added to the candidate class, and the same empirical risk minimisation (ERM) argument can be applied together with a VC dimension bound for the enlarged class. Thus approximation and clipping errors specific to each model enter through the comparator risk, whereas learning contributes the usual complexity penalty.

%
%

\section{Simulation}
\label{sec:generalization}

\begingroup
This section examines the oracle comparison in finite samples. The experiment uses fixed CNN constructions for the slope change, realised volatility and autoregressive statistics, together with a joint affine head that learns how to combine their evidence. The objective is to compare the joint classifier with the best fixed statistical rule and to study how the comparison changes with the training sample size. Two additional numerical experiments in the Supplementary Material provide sampled checks of the exact representation and uniform approximation constructions.

\subsection{Data generation}
All trajectories have length \(T=40\), and the two classes are equally likely. We consider three scenarios aligned with the model settings above and one heterogeneous scenario.

In the slope change scenario, $X_t=A+Bt+\varepsilon_t-
Y D\frac{\kappa_S}{2}|t-\tau|, t=1,\ldots,T,$ where \(A\sim\operatorname{Unif}[-0.05,0.05]\),
\(B\sim\operatorname{Unif}[-0.008,0.008]\),
\(\varepsilon_t\stackrel{\mathrm{iid}}{\sim}N(0,0.01^2)\),
\(D\) is uniform on \(\{-1,1\}\), \(\tau\) is uniform on
\(\{18,\ldots,22\}\), and \(\kappa_S=0.03\). Thus the nuisance linear trend is present in both classes, whereas the alternative adds a continuous peak or valley whose two segment slopes differ by \(\kappa_S\). In the volatility change scenario, \(X_1=0\) and
\(X_{t+1}=X_t+r_t\), where
\(r_t=(T-1)^{-1}A_tD_t\), the \(D_t\)'s are independent Rademacher
variables, and the \(A_t\)'s are independent Bernoulli variables with success probability \(0.20\) under the null and \(0.55\) under the alternative. In the AR scenario, \(X_1=\varepsilon_1\) and
\(X_t=\Phi X_{t-1}+\varepsilon_t\), where
\(\varepsilon_t\stackrel{\mathrm{iid}}{\sim}N(0,0.03^2)\),
\(\Phi=1\) under the null and \(\Phi=1.10\) under the alternative. Both classes are divided by the common scale \(4(0.03)\{\sum_{j=0}^{T-1}1.10^{2j}\}^{1/2}\). 

The heterogeneous scenario uses a common Gaussian random walk null generated from \(X_0=0\) and \(X_t=X_{t-1}+\varepsilon_t\), where \(\varepsilon_t\overset{\mathrm{iid}}{\sim}N(0,0.01^2)\), for \(t=1,\ldots,T\). Under the alternative, one of three subtypes is selected independently with probability \(1/3\). The slope subtype adds \(-D(0.02)|t-\tau|/2\) to this path, where \(D\) is uniform on \(\{-1,1\}\) and \(\tau\) is uniform on \(\{18,\ldots,22\}\). The volatility subtype instead adds \(0.04A_{t-1}D_{t-1}\) to each increment for \(t\geq2\), where the \(A_t\)'s are independent Bernoulli variables with success probability \(0.40\) and the \(D_t\)'s are independent Rademacher variables. The AR subtype replaces the random walk by \(X_1=\varepsilon_1\) and \(X_t=1.10X_{t-1}+\varepsilon_t\). The subtype, locations, directions, jumps and innovations are mutually independent. Thus all null trajectories have the same distribution, and heterogeneity arises only under the alternative.

\subsection{Calibration and learning procedures}
For each scenario, an independent null sample of size \(20{,}000\) is generated before the training data. It fixes the \(0.95\) null quantiles \(\lambda_S,\lambda_V,\lambda_{AR}\), the feature scales, the clipping radius \(M=1\), and the triangular map approximants at level \(m=9\) used to approximate the quadratic statistics. The scales $s_S,s_V,s_{AR}$ are the sample standard deviations in this calibration sample of the respective raw margins $S-\lambda_S$, $\widetilde V-\lambda_V-\epsilon_V$, and $\widetilde S_{AR}-\lambda_{AR}-\epsilon_{AR}$, with any value smaller than $10^{-8}$ replaced by $10^{-8}$. At \(T=40\), their analytic allowances for uniform error are \(\epsilon_V=1.488\times10^{-4}\) and \(\epsilon_{AR}=1.116\times10^{-4}\). The three fixed comparators are
\(\overline h_S=\mathbb I\{S>\lambda_S\}\),
\(\overline h_V=\mathbb I\{\widetilde V>\lambda_V+\epsilon_V\}\), and \(\overline h_{AR}=\mathbb I\{\widetilde S_{AR}>
\lambda_{AR}+\epsilon_{AR}\}\), where the tildes denote the fixed neural approximations applied after clipping. All their specifications are therefore fixed before the labelled training sample, as required in Theorem~\ref{thm:three_rule_oracle}.

We compare three quantities. First, the oracle benchmark is the smallest test error among the three fixed comparators. It is infeasible because its identity is determined using the test sample. Within each Monte Carlo replication, this selection is made once for the entire replication. Second, the ERM selector for a single comparator evaluates the three comparators on the labelled training sample, selects one using training error alone, and then reports the test error of that selected comparator. Third, the joint classifier with an affine head uses the externally standardised margins
\[
Z(\bm X)=\left(
\frac{S(\bm X)-\lambda_S}{s_S},
\frac{\widetilde V(\bm X)-\lambda_V-\epsilon_V}{s_V},
\frac{\widetilde S_{AR}(\bm X)-\lambda_{AR}-\epsilon_{AR}}{s_{AR}}
\right)
\]
and predicts \(\mathbb I\{\widehat\beta_0+
\widehat{\bm\beta}^{\mathsf T}Z(\bm X)>0\}\). The coefficients are fitted by logistic cross entropy with an \(L_2\) penalty. This is a computationally practical surrogate for ERM in the common class. Unlike the baseline using a single comparator, it can combine information from all three branches.

For each Monte Carlo replication, we generate nested balanced training samples with \(N\in\{200,500,1000\}\) and an independent balanced test sample of size \(4{,}000\). The same test sample is used for all three training sizes within a replication, and the procedure is repeated \(500\) times. Approximation level, clipping, thresholds and normalisation scales remain fixed as \(N\) varies. For a fitted classifier \(\widehat h\), define its difference in test error from the oracle benchmark by $\widehat E_{\mathrm{oracle}}(\widehat h)
=\widehat R_{\mathrm{test}}(\widehat h)
-\min_{G\in\{S,V,AR\}}
\widehat R_{\mathrm{test}}(\overline h_G).$

\subsection{Oracle comparison results}
Table~\ref{tab:oracle_simulation} reports the Monte Carlo mean of the principal performance measures. The oracle column is constant across \(N\) because the comparators and the paired test sample do not depend on the training size. The small difference between the oracle and ERM selector in the slope scenario is a selection effect due to the finite sample. In the volatility and mixed scenarios, the equality of the oracle and entries for the single ERM selector across \(N\) means that the selector based on training and the oracle based on the test set identify the same fixed rule throughout the experiment. Realised volatility is selected as the oracle comparator in the mixed scenario based on its test set performance. The joint affine head has lower test error than the oracle comparator in the slope, volatility and mixed scenarios. In the AR scenario it is slightly worse at \(N=200\), approximately matches the comparator at \(N=500\), and improves on it at \(N=1000\). The absolute gap between training and test errors decreases with \(N\) in every scenario, while the AUC remains high.

\begin{table}[H]
\centering
\small
\setlength{\tabcolsep}{10.2pt}
\renewcommand{\arraystretch}{0.5}
\caption{Oracle comparison simulation results. Entries are Monte Carlo means over \(500\) replications. Abs. gap is the absolute gap between training and test errors of the joint classifier.}
\label{tab:oracle_simulation}
\begin{tabular}{llrrrrr}
\toprule
Scenario & \(N\) & Oracle & Single ERM & Joint affine & Abs. gap & AUC \\
\midrule
\multirow{3}{*}{Slope}
& 200  & 0.0241 & 0.0260 & 0.0155 & 0.0106 & 0.9965 \\
& 500  & 0.0241 & 0.0255 & 0.0129 & 0.0049 & 0.9976 \\
& 1000 & 0.0241 & 0.0251 & 0.0121 & 0.0031 & 0.9979 \\
\midrule
\multirow{3}{*}{Volatility}
& 200  & 0.0185 & 0.0185 & 0.0131 & 0.0097 & 0.9990 \\
& 500  & 0.0185 & 0.0185 & 0.0106 & 0.0047 & 0.9994 \\
& 1000 & 0.0185 & 0.0185 & 0.0099 & 0.0029 & 0.9995 \\
\midrule
\multirow{3}{*}{AR}
& 200  & 0.0696 & 0.0707 & 0.0708 & 0.0157 & 0.9608 \\
& 500  & 0.0696 & 0.0696 & 0.0688 & 0.0095 & 0.9623 \\
& 1000 & 0.0696 & 0.0696 & 0.0682 & 0.0069 & 0.9628 \\
\midrule
\multirow{3}{*}{Mixed}
& 200  & 0.0604 & 0.0604 & 0.0552 & 0.0151 & 0.9735 \\
& 500  & 0.0604 & 0.0604 & 0.0528 & 0.0089 & 0.9741 \\
& 1000 & 0.0604 & 0.0604 & 0.0521 & 0.0060 & 0.9742 \\
\bottomrule
\end{tabular}
\end{table}

Figure~\ref{fig:oracle_simulation} displays the paired test error differences. The curve for the single ERM selector remains close to the zero oracle line in all four panels. The curve for the joint affine head lies below zero for every training size in the slope, volatility and mixed scenarios, and its advantage increases with \(N\). These patterns show that the learned affine head can improve on the fixed calibrated comparators.

\begin{figure}[h]
\centering
\includegraphics[width=\linewidth]{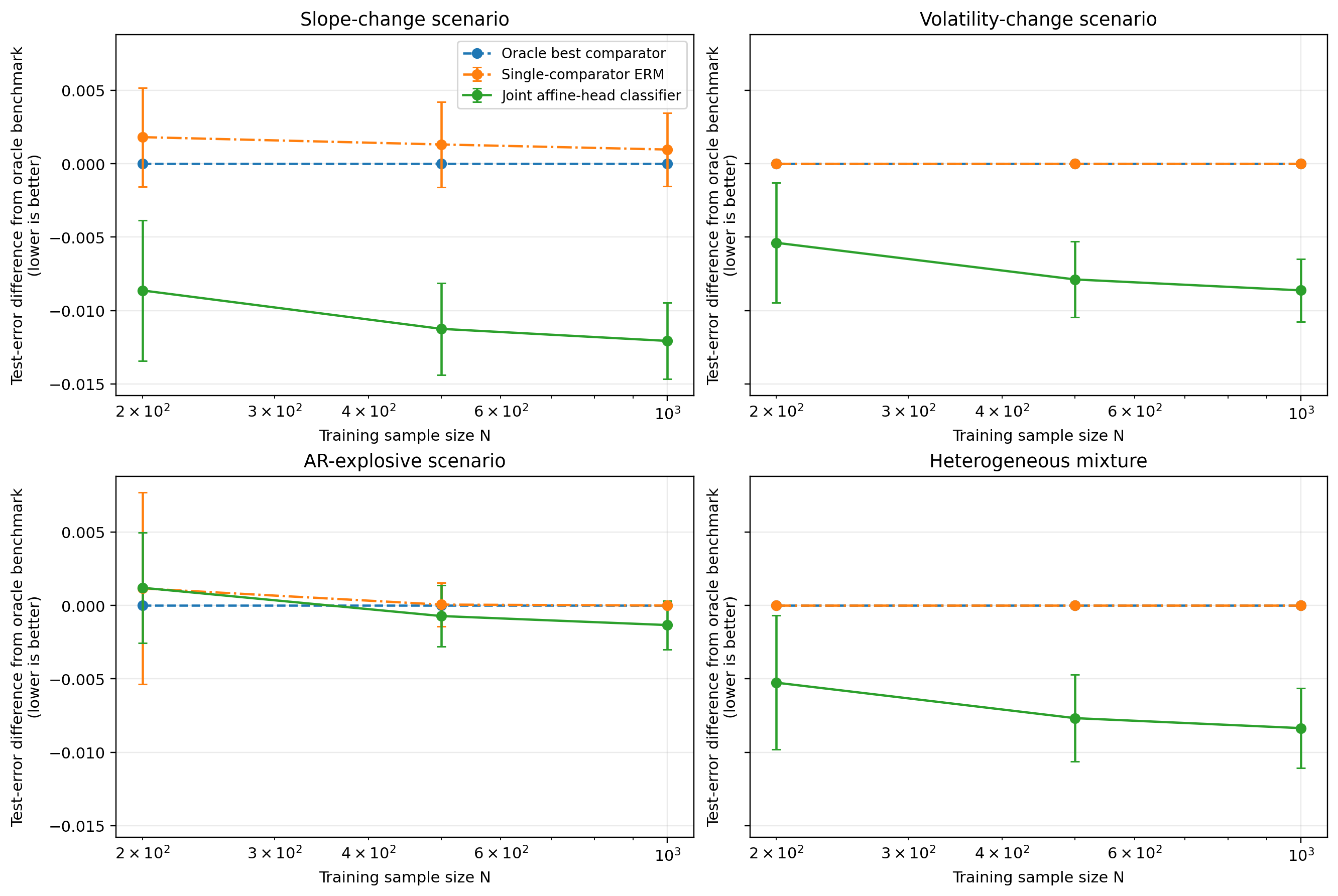}
\caption{Test error differences from the oracle best comparator. Points are Monte Carlo means and error bars are one Monte Carlo standard deviation over \(500\) replications. The oracle benchmark is zero by construction, and negative values indicate an improvement over every fixed comparator.}
\label{fig:oracle_simulation}
\end{figure}

\endgroup

\section{Empirical results}\label{Application}

\subsection{Energy price data}
We use six daily energy commodity spot price series from the \href{https://fred.stlouisfed.org/}{Federal Reserve Economic Data (FRED)} database. The dataset covers several major categories of energy commodities, including crude oil, natural gas, and refined petroleum products such as gasoline, diesel and jet fuel. Table \ref{tab:data detail} provides a brief description of the commodity price series used in the analysis.

\begin{table}[ht]
\centering
\setlength{\tabcolsep}{19pt}
\caption{Description of the commodity price series used in the analysis.}
\small
\begin{tabular}{ll}
\toprule
Commodity Name & Detail \\
\midrule
DCOILBRENTEU & Brent crude oil prices in Europe  \\
DHHNGSP & Natural gas spot prices at Henry Hub\\
DGASNYH & Conventional gasoline prices in New York Harbor \\
DRGASLA & Reformulated gasoline blendstock prices in Los Angeles \\
DDFUELUSGULF & Ultra-low-sulfur No. 2 diesel fuel prices on the U.S. Gulf Coast \\
DJFUELUSGULF & Kerosene jet fuel prices on the U.S. Gulf Coast \\

\bottomrule
\end{tabular}
\label{tab:data detail}
\end{table}

To investigate whether explosive price patterns in energy commodities are associated with events in the real world, we classify historical events into several representative categories based on historical records, financial news sources, reports from the \href{https://www.eia.gov/}{U.S. Energy Information Administration (EIA)}, and AI tools for information retrieval, including GPT-5.5, Gemini 3.5, and the \href{https://ask.ft.com/}{Financial Times AI assistant}. We use these AI tools to identify candidate events, summarise their economic relevance, and collect structured event attributes, including the affected commodity, event category, event start and end dates, a key event date, information sources, and a short description of the underlying shock. The key event date represents the most important date associated with the event, such as an announcement date, outbreak date, or policy decision date.

Specifically, we consider three event categories according to the dominant mechanism through which each event is expected to affect the corresponding energy price. First, weather and natural hazard events (Weather) are initiated by meteorological conditions or natural disasters, such as hurricanes, winter storms, and extreme cold spells. Second, geopolitical and security events (Geopolitical) comprise armed conflicts and major security shocks, including the Gulf War and Russia's invasion of Ukraine. The combined supply and financial family (Supply-financial) contains the remaining supply, production, policy, sanctions, operational and macroeconomic or financial shocks. Representative examples include the 2014--2016 collapse in oil prices associated with the U.S. shale boom and OPEC output policy, and the Global Financial Crisis. The resulting records were subsequently screened and curated, with manual checks focusing on duplicate entries, category consistency, and source availability.

\subsection{Model training}\label{sec:model_training}

The empirical procedure has two stages. Stage~1 trains a separate CNN event detector for each commodity. Stage~2 trains two pooled binary CNN classifiers on windows with known labels for event families. The first separates weather events from the other families, and the second separates geopolitical events from supply-financial events. During prediction, only windows classified as events by Stage~1 enter this hierarchy, producing four final classes consisting of no event and the three event families.

Each binary classifier uses one branch of the architecture with convolution, pooling and a classifier head in Section~\ref{sec:cnn}. Three convolutional blocks with kernels of size three map the channel dimension from $1$ to $32$, $64$ and $128$. Each block applies a ReLU activation and batch normalisation. The first two use max pooling, while the third uses adaptive max pooling. A linear layer maps the resulting $128$ features to $64$ hidden units, followed by ReLU, dropout at rate $0.3$ and an output layer producing two scores. Softmax converts the scores into class probabilities. Dropout randomly masks hidden units only during training, while prediction uses the full fitted network. At prediction, batch normalisation uses the running means and variances estimated during training and therefore acts as a fixed channel-wise affine transformation.

To reduce class imbalance in Stage~1, the non-event class in the training set is randomly downsampled, retaining $\min(N_0,rN_1)$ of the $N_0$ non-event windows, where $N_1$ is the number of event windows and $r=1.5$. Downsampling is applied only to the training set, so the validation and test sets retain their original class proportions.

The Stage~1 classifiers are then trained with the focal loss of \citet{lin2017focal}, which reduces the influence of easily classified observations. For a minibatch $\mathcal{B}$,
\begin{equation}
\mathcal{L}
=\frac{1}{|\mathcal{B}|}\sum_{i\in\mathcal{B}}
-\alpha_{y_i}(1-p_{y_i})^{\nu}\log(p_{y_i}),
\label{eq:fl}
\end{equation}
where $p_{y_i}$ is the softmax probability assigned to the true class $y_i$ and $\nu=2$, following the default specification of \cite{lin2017focal}. For Stage~1, $\alpha_{y_i}=N/(2N_{y_i})$, where $N_{y_i}$ is the number of observations in class $y_i$ and $N$ is the size of the training set. Parameters are updated with the Adam optimiser \citep{kingma2015adam}. The Stage~2 classifiers use the same focal term without downsampling. Let $g(i)$ denote the event of window $i$, let $n_g$ be the number of training windows from event $g$, and let $G_c$ be the number of events in classifier class $c$. Each window first receives base weight $1/n_{g(i)}$, so the weights sum to one within every event. The base weights therefore sum to $G_c$ in class $c$ and are multiplied by $a_c=(G_0+G_1)/(2G_c)$. Up to a common normalisation, the final weight of window $i$ is $w_i=a_{y_i}/n_{g(i)}=(G_0+G_1)/(2G_{y_i}n_{g(i)})$, and the Stage~2 objective is $\sum_{i\in\mathcal B}w_i[-(1-p_{y_i})^\nu\log(p_{y_i})]/\sum_{i\in\mathcal B}w_i$. The two classifier classes consequently have equal total weight and, within each class, every event has equal total weight. Thus neither events producing more overlapping windows nor classes containing more events dominate training. For each Stage~2 classifier, the checkpoint with the highest validation F1 score is retained.

The rolling window length is selected by validation performance. For each $T\in\mathcal T_{\mathrm{cand}}$, the windows and labels are rebuilt before the split. We select $T^*=\arg\max_{T\in\mathcal T_{\mathrm{cand}}}\mathrm{F1}_{\mathrm{val}}(T)$ and evaluate only the associated model on the test set. Algorithm~\ref{algo:1} summarises the Stage~1 procedure.

\begingroup
\setstretch{1.5}
\begin{algorithm}[H]
\caption{Labelling, training and window length selection}
\label{algo:1}
\begin{algorithmic}[1]
\Require Price series $\{X_1,\ldots,X_n\}$, event table $\mathcal E$, candidate window lengths $\mathcal T_{\mathrm{cand}}$
\Ensure Selected window length $T^*$ and trained CNN $f_{\theta^*}$
\State Remove events with impact measure $q_j<0.2$
\State Initialize $F^*\leftarrow-\infty$
\For{each $T\in\mathcal T_{\mathrm{cand}}$}
\State Construct rolling windows $\mathbf W_{t,T}\in\mathbb R^{1\times T}$, $t=T,\ldots,n$
\State Set $y_{t,T}=1$ if $r_{tj}\geq\tau$ or $d_j\in\mathbf W_{t,T}$ for some $E_j$, and set $y_{t,T}=0$ otherwise
\State Split the labelled windows into training, validation and test sets
\State Downsample the training set majority class and compute class weights
\State Train the CNN with the loss in equation~(\ref{eq:fl}) and denote the parameters by $\theta_T$
\If{$\mathrm{F1}_{\mathrm{val}}(T)>F^*$}
\State $F^*\leftarrow\mathrm{F1}_{\mathrm{val}}(T)$, $T^*\leftarrow T$, $\theta^*\leftarrow\theta_T$
\EndIf
\EndFor
\State Evaluate $f_{\theta^*}$ with window length $T^*$ on the test set
\end{algorithmic}
\end{algorithm}
\endgroup

\subsection{Data labelling}\label{sec:data_labelling}

For each commodity price series $\bm{X}=(X_1,\ldots,X_n)$ and candidate window length $T$, we construct rolling windows of length $T$. Let window $i$ be denoted by $W_i=[s_i,e_i]$, where $s_i$ and $e_i$ are its start and end dates. Each historical event is represented by an interval $E_j=[a_j,b_j]$, with key event date $d_j$.

Before assigning labels to rolling windows, we apply an impact filter to each event to remove events with limited price movement during the event period. For event $E_j$ and the associated segment $(X_{a_j},X_{a_j+1},\ldots,X_{b_j})$ of the commodity price series, we compute
\begin{equation*}
    q_j=\frac{\max_{t \in E_j}X_t-\min_{t\in E_j}X_t}{|\frac{1}{|E_j|}\sum_{t \in E_j}X_t|}.
\end{equation*}
Events with $q_j <0.2$ are excluded from the labelling procedure. This step reduces the influence of events that are historically relevant but do not correspond to a substantial observable price movement in the commodity series.

To determine whether a window is substantially associated with an event, we compute the overlap ratio
\begin{equation}
    r_{ij}=\frac{|W_i \cap E_j|}{\min \{|W_i|,|E_j|\}}.
    \label{eq:overlap ratio}
\end{equation}
The window $W_i$ is labelled as an event window if there exists an event interval $E_j$ such that $r_{ij} \ge \tau$ or $d_j \in W_i$, which means it either overlaps sufficiently with at least one event interval or contains the key event date. Otherwise it is labelled as a non-event window. In the empirical implementation, we set $\tau = 0.3$. Representative examples of commodity price trajectories and corresponding event periods are illustrated in Figure \ref{fig:data example}.

Throughout the splitting procedure, the same historical event associated with different commodities is treated as a separate event for each commodity. In both stages, all windows linked to an event are assigned to a single subset, and boundary windows are removed until no raw price observation is shared across subsets. Stage~1 is split separately for each commodity. Windows are first ordered by end date and assigned initially to training, validation and test sets in proportions of $56\%$, $24\%$ and $20\%$. The initial assignments are then adjusted to satisfy the restrictions above.

Stage~2 is constructed separately from labelled event windows pooled across commodities. The test events are selected only from the Stage~1 test set, having $20\%$ of events within each event family, combining eligible windows without events from the Stage~1 test set. Validation events are then selected from the remaining events outside that pool whenever possible, having $20\%$ of events within each event family. The remaining $60\%$ of events within each event family form the Stage~2 training set. The final joint test set is then constructed from the $9{,}001$ Stage~1 test windows. It retains the Stage~2 test event windows and the no-event windows that share no raw price observations with Stage~2 training or validation events. Removing $1{,}881$ event windows not assigned to Stage~2 testing and $1{,}254$ overlapping no-event windows leaves $5{,}866$ windows for the joint evaluation.

\begin{figure}[htbp]
    \centering
    \includegraphics[width=0.68\linewidth]{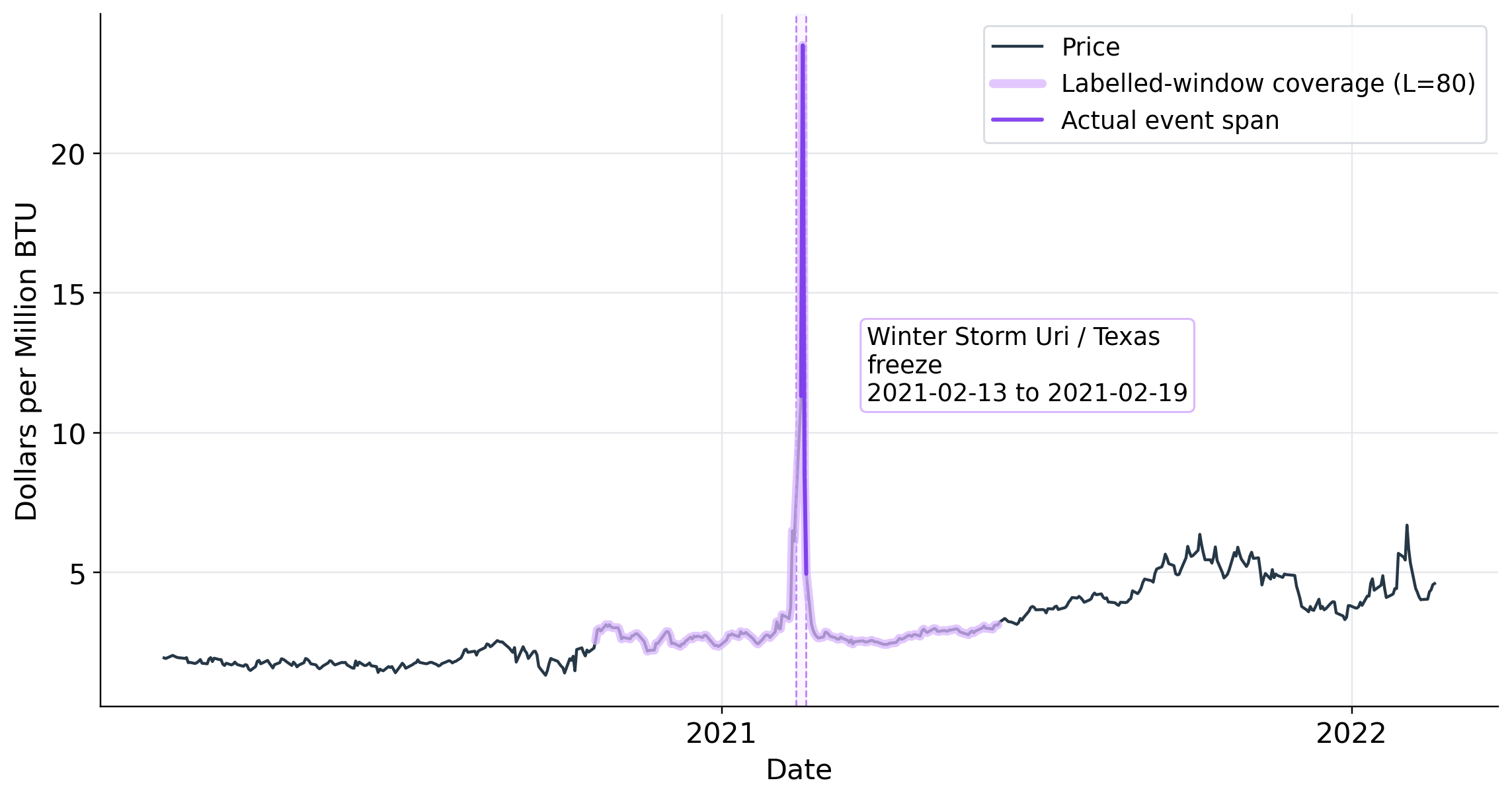}
    \includegraphics[width=0.68\linewidth]{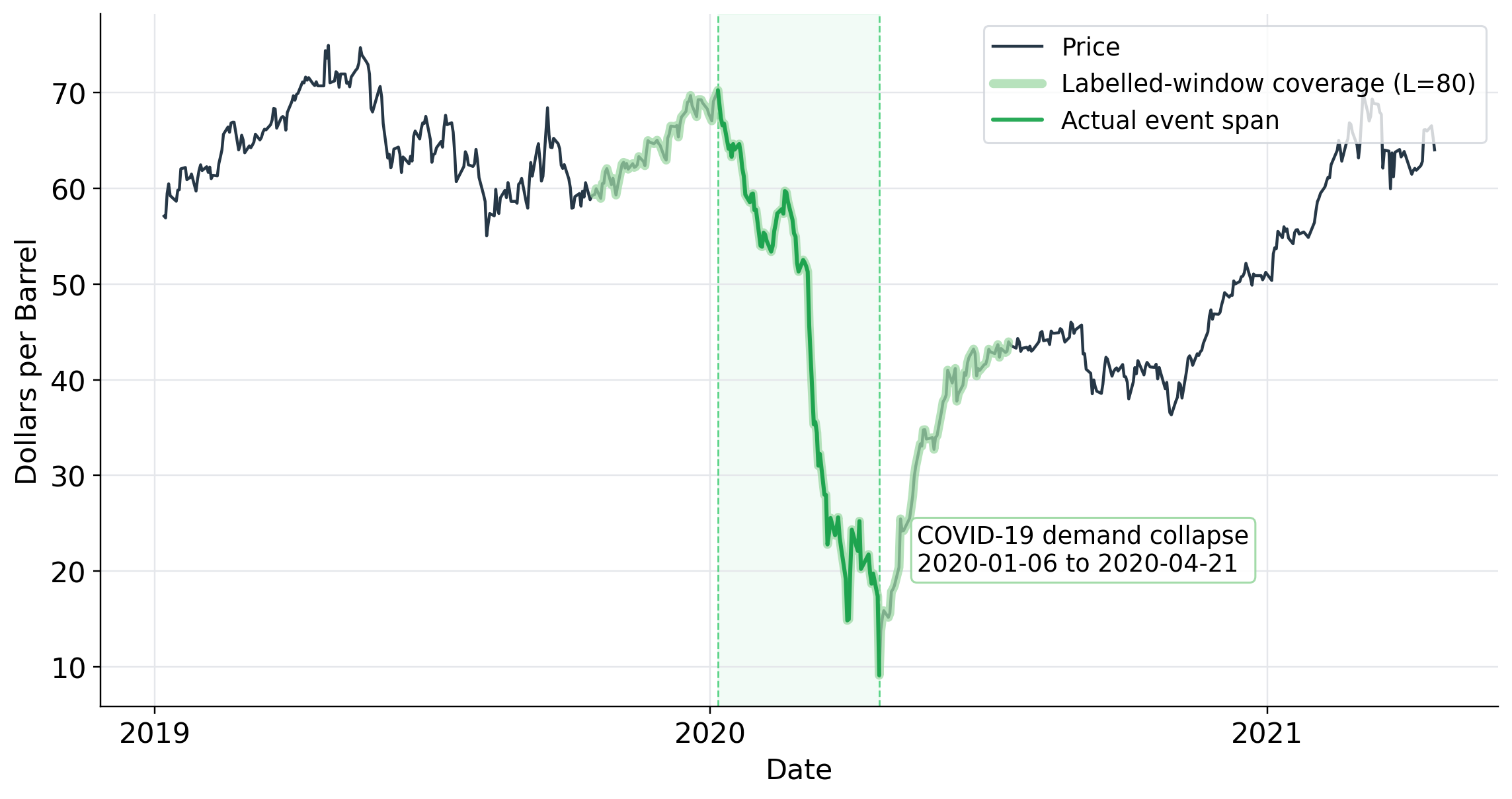}
    \includegraphics[width=0.68\linewidth]{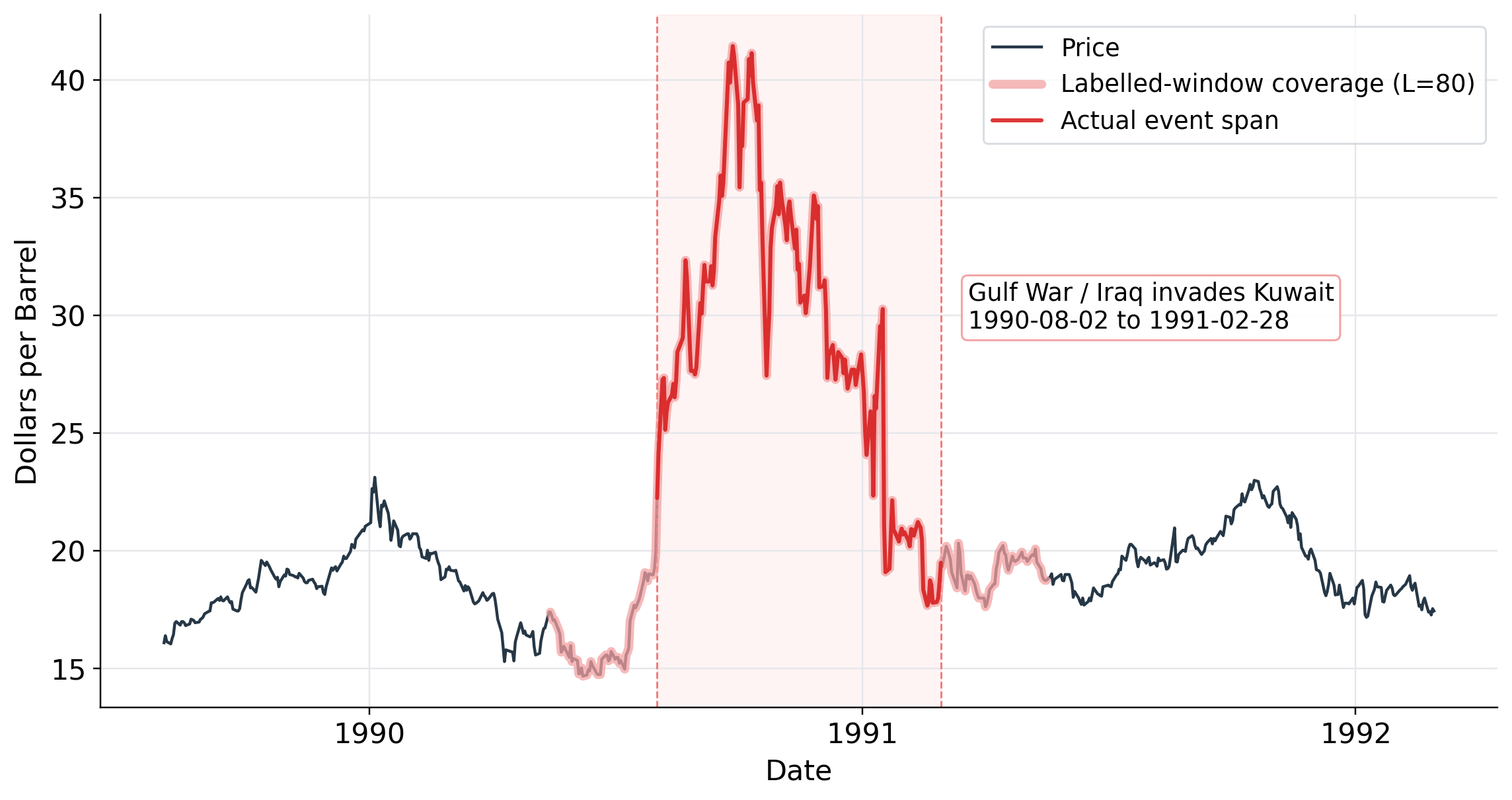}
    \caption{Representative examples of Henry Hub natural gas and Brent crude oil prices under different event categories. The highlighted curve segment represents the actual event period, while the pale thick curve segment represents the price observations covered by rolling windows labelled as event windows. From top to bottom, the panels show Winter Storm Uri for Henry Hub natural gas (Weather), the COVID-19 demand collapse for Brent crude oil (Supply-financial), and the Gulf War for Brent crude oil (Geopolitical).}
    \label{fig:data example}
\end{figure}

\subsection{Binary classification results}
This section reports the Stage~1 binary event detection results for each commodity and compares the proposed CNN with alternative prediction procedures.

Table \ref{tab:binary results} reports the test performance of the six CNNs trained separately by commodity. Henry Hub natural gas attains the highest accuracy, F1 and AUC, while Brent crude oil has the lowest accuracy and AUC. Diesel has almost complete event recall but lower precision, indicating that its main errors are false event detections.

\begin{table}[ht]
\centering
\setlength{\tabcolsep}{16pt}
\caption{Binary classification performance for the CNN trained for each of the six commodities.}
\small
\begin{tabular}{lccccc}
\toprule
Commodity & Accuracy & Precision & Recall & F1 & AUC \\
\midrule
DGASNYH & 0.8753 & 0.7964 & 0.7472 & 0.7710 & 0.9362 \\
DHHNGSP & 0.9189 & 0.9257 & 0.8220 & 0.8708 & 0.9655 \\
DDFUELUSGULF & 0.7392 & 0.6550 & 0.9959 & 0.7903 & 0.9370 \\
DRGASLA & 0.8314 & 0.7572 & 0.8859 & 0.8165 & 0.9096 \\
DJFUELUSGULF & 0.7980 & 0.6221 & 0.9320 & 0.7461 & 0.9176 \\
DCOILBRENTEU & 0.7362 & 0.6667 & 0.8759 & 0.7571 & 0.8427 \\
\bottomrule
\end{tabular}
\label{tab:binary results}
\end{table}

\begin{figure}[ht]
       \centering
       \includegraphics[width=0.49\linewidth]{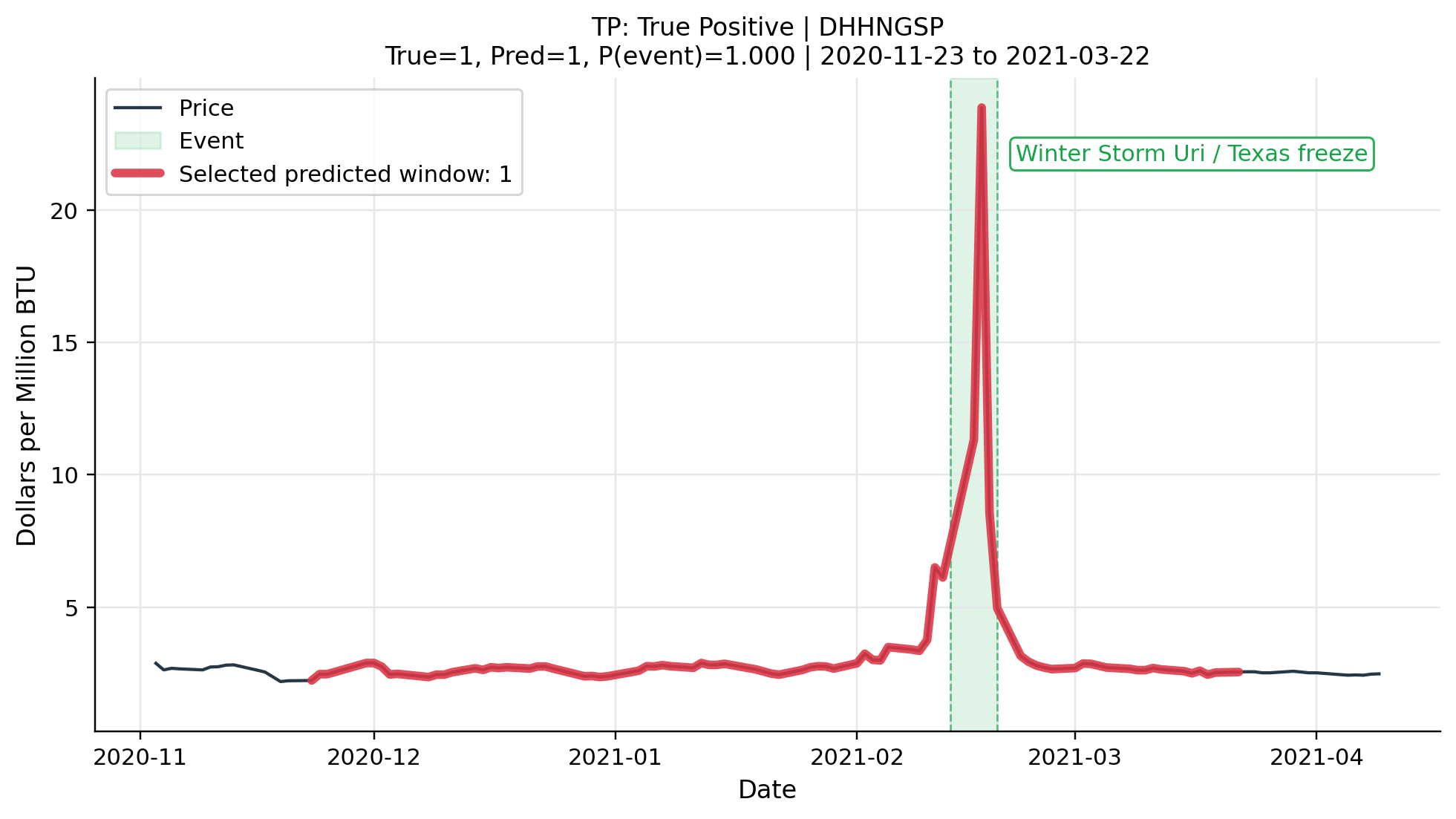}
       \includegraphics[width=0.49\linewidth]{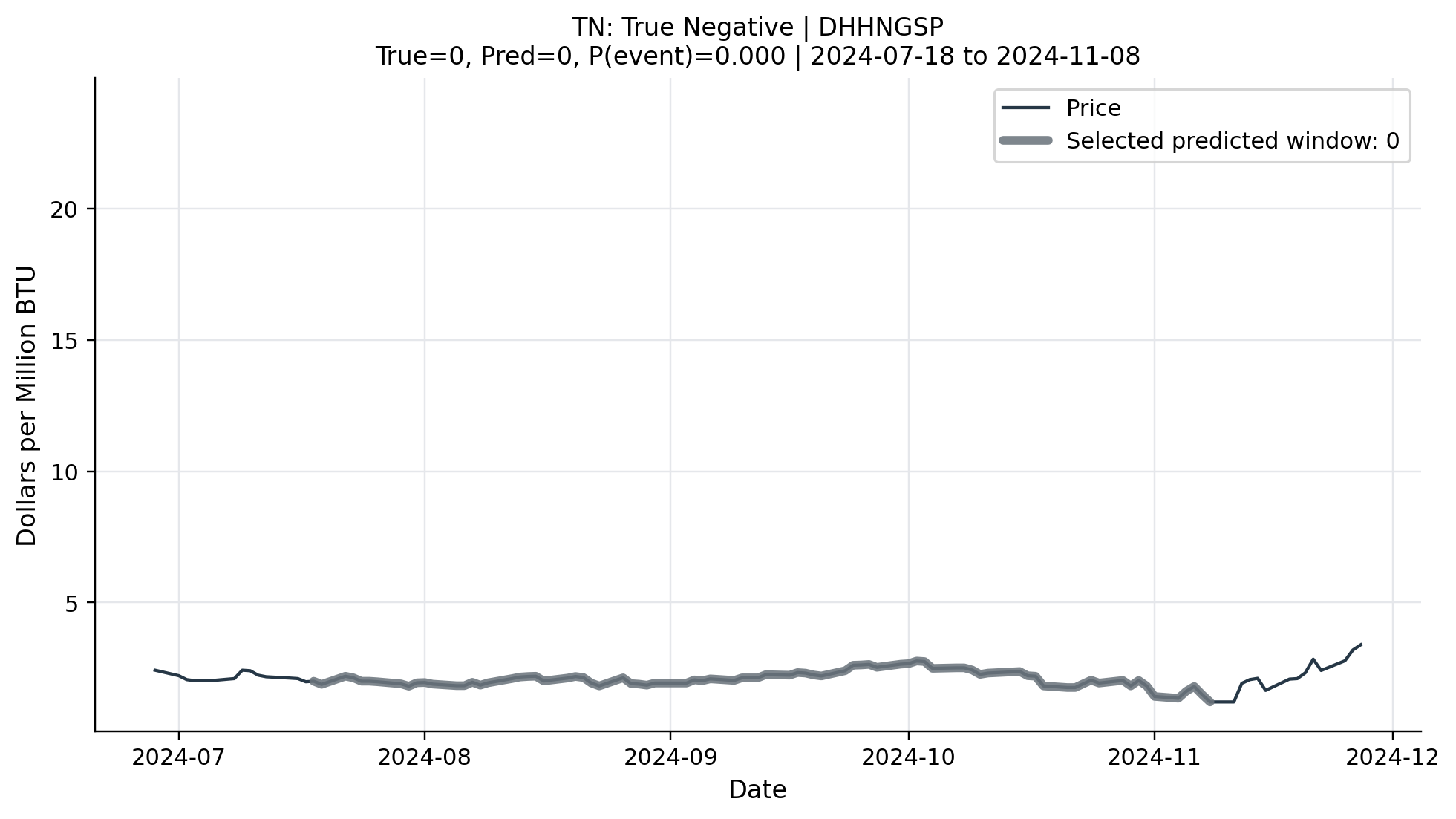}
       \includegraphics[width=0.49\linewidth]{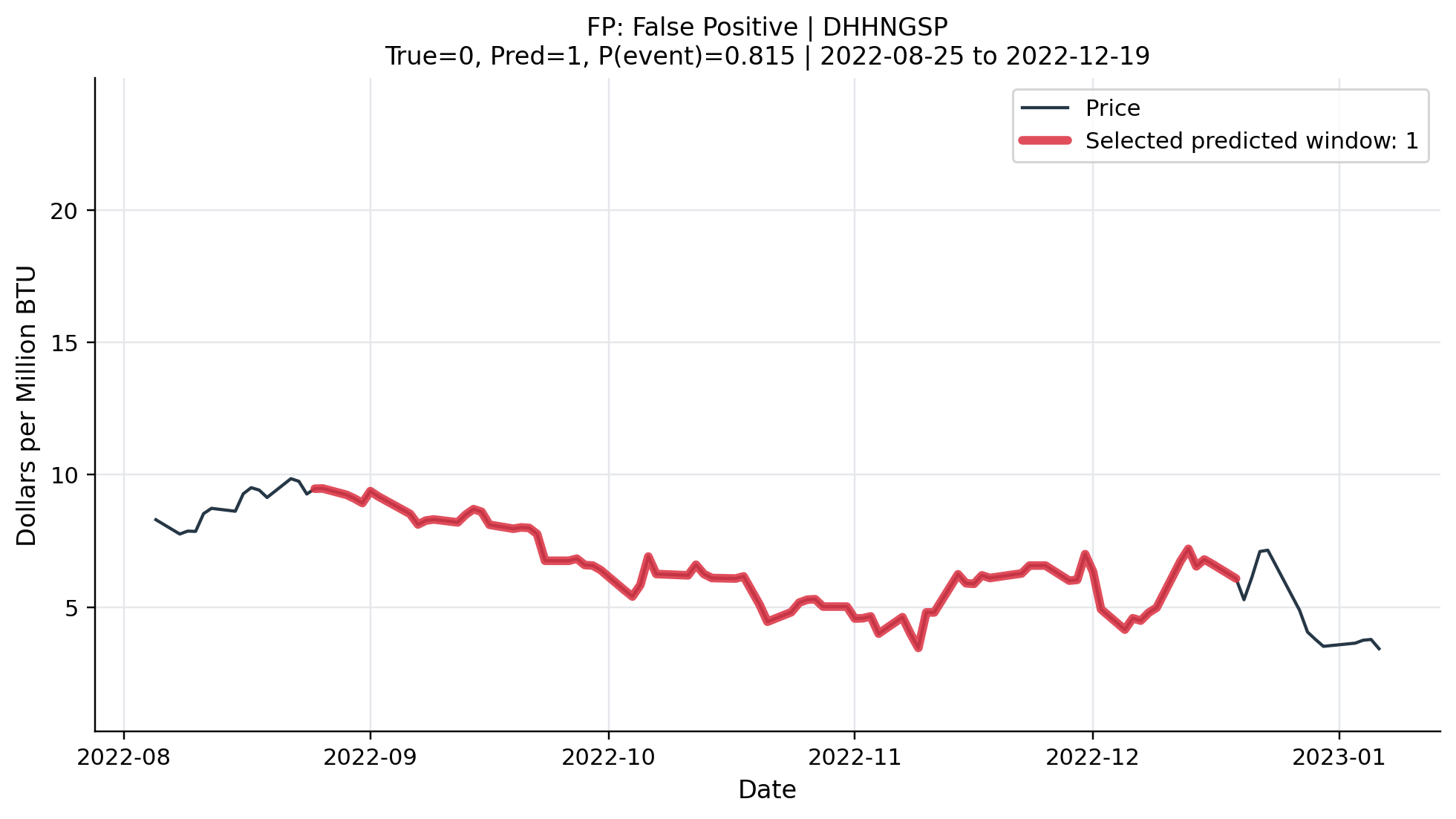}
       \includegraphics[width=0.49\linewidth]{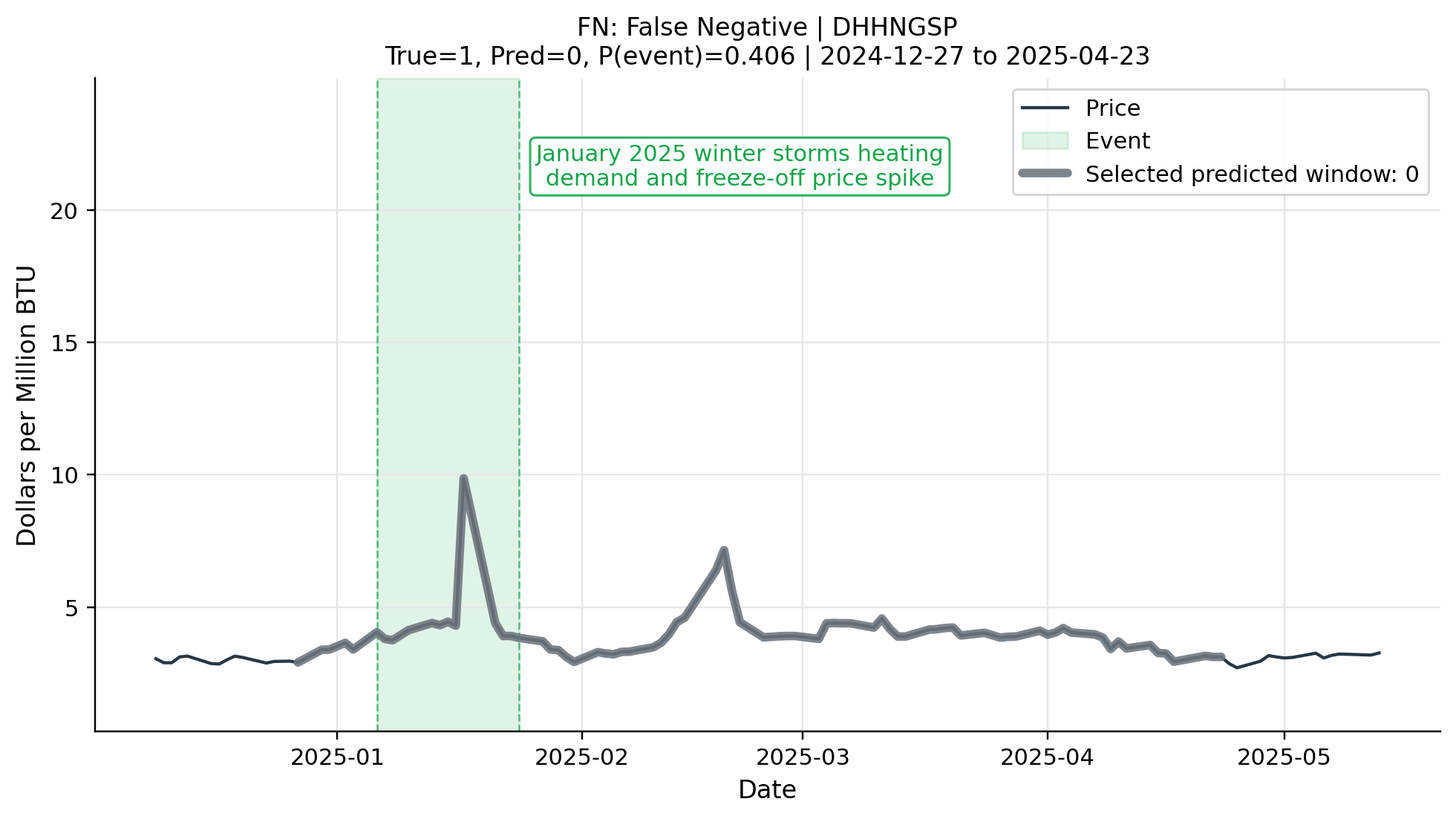}
       \caption{Representative CNN classifications for the Henry Hub natural gas price series. The panels show a true positive, true negative, false positive and false negative, respectively.}
       \label{fig:henry_example}
\end{figure}

In addition, Figure~\ref{fig:henry_example} presents several representative binary classification results for the Henry Hub natural gas price series, including examples of true positives (TP), true negatives (TN), false positives (FP), and false negatives (FN). The TP and TN examples show that the CNN model is able to distinguish explosive event-driven price patterns from relatively stable non-event regimes.

For the FP case, although the selected window is not labelled as an event window, the price trajectory still exhibits local fluctuations, suggesting that the model may capture patterns resembling events beyond the manually labelled intervals. For the FN example, the event period contains multiple spikes and more complicated local structures, which may reduce the model's confidence and lead to misclassification. However, it still assigns a probability of $0.406$ to the event class, which is relatively close to the decision threshold for event detection.

Table~\ref{tab:model_comparison} compares the proposed CNN with a ResNet CNN, logistic regression, MLPs with one and two layers, a block sparse MLP, an LSTM and an FT-Transformer across the six commodity series. Detailed descriptions of the ResNet CNN and the block sparse MLP are provided in the Supplementary Material.

\begin{table}[htbp]
\setlength{\tabcolsep}{16pt}
\centering
\caption{Unweighted mean of the binary classification metrics computed separately for each commodity across six commodity series.}
\label{tab:model_comparison}
\small
\begin{tabular}{lccccc}
\toprule
Model & Accuracy & Precision & Recall & F1 & AUC \\
\midrule
CNN & 0.8165 & 0.7372 & 0.8765 & 0.7920 & 0.9181 \\
ResNet CNN & 0.7114 & 0.6111 & 0.7582 & 0.6561 & 0.8456 \\
Logistic regression & 0.5077 & 0.4672 & 0.7500 & 0.5217 & 0.7161 \\
MLP (one layer) & 0.4976 & 0.4633 & 0.8204 & 0.5576 & 0.7270 \\
MLP (two layers) & 0.5511 & 0.4962 & 0.7910 & 0.5805 & 0.7791 \\
Block sparse MLP & 0.5144 & 0.5041 & 0.8439 & 0.5793 & 0.7894 \\
LSTM & 0.4925 & 0.5532 & 0.6873 & 0.4616 & 0.6887 \\
FT-Transformer & 0.4851 & 0.5005 & 0.7583 & 0.5138 & 0.7938 \\
\bottomrule
\end{tabular}
\end{table}

All methods use the same labelled windows and fixed Stage~1 split memberships. Each method is fitted separately for each commodity using the same deterministically downsampled training windows and is evaluated on the same test set. The proposed CNN uses focal loss with class weights, the neural baselines use cross entropy with class weights, and logistic regression uses a balanced logistic objective. Reported values are unweighted means of the test metrics for the six commodities.

The proposed CNN attains the highest accuracy, precision, recall, F1 and AUC. Logistic regression and the MLPs do not explicitly share local features across time, which may make transient event patterns harder to distinguish from ordinary fluctuations. The LSTM and FT-Transformer can represent dependence over longer ranges, but their greater flexibility may be difficult to estimate from the limited number of distinct labelled events because overlapping windows do not provide equally many independent episodes. The deeper ResNet CNN preserves temporal locality but may introduce more complexity than is useful for these relatively short windows. The out-of-sample performance is compared on a common empirical basis. Loss and optimisation choices remain method specific, and the table therefore compares complete prediction procedures rather than isolating the effect of architecture alone.

\subsection{Multiclass classification results}

We now report the results of the hierarchical procedure described in Sections~\ref{sec:model_training} and~\ref{sec:data_labelling}. Stage~1 supplies the event detector fitted to each commodity, and its positive predictions are routed through the two Stage~2 classifiers.

Table~\ref{tab:multi_stagewise} reports conditional classifier performance and routing quality. For Stage~2 classifier \(j\), let \(\mathcal R_j\) be the set of routed test windows, let \(\mathcal S_j\subseteq\mathcal R_j\) contain the windows within that classifier's intended scope, and let \(C_j=\sum_{i\in\mathcal S_j}\mathbb I\{\widehat Y_{ij}=Y_i\}\). Classification metrics are computed on \(\mathcal S_j\), route purity is \(|\mathcal S_j|/|\mathcal R_j|\), and effective accuracy is \(C_j/|\mathcal R_j|\). The latter treats every routed window outside the intended scope as an error. For comparability with the complete sequential classifier, the Stage~1 row pools predictions from the six commodity-specific event detectors on the final joint test set of \(5{,}866\) windows. The weather classifier performs well within its scope, but routing errors reduce its effective performance. The classifier separating geopolitical and supply-financial events is the main bottleneck, reflecting both weaker separation between these families and propagated Stage~1 errors.

\begin{table*}[htbp]
\centering
\caption{Classification performance at each stage and routing quality of the hierarchical classifier. The column ``Eval. \(n\)'' gives the number of test windows underlying each component's classification metrics. For each Stage~2 classifier, it equals the number of in-scope windows.}
\label{tab:multi_stagewise}
\setlength{\tabcolsep}{3pt}
\renewcommand{\arraystretch}{1.15} 
{\small
\begin{tabular}{@{}lcccccc@{\hspace{\tabcolsep}\vrule width \arrayrulewidth\hspace{\tabcolsep}} ccc@{} }
\toprule
&
\multicolumn{6}{ c@{\hspace{\tabcolsep}\vrule width \arrayrulewidth\hspace{\tabcolsep}} }{Classification performance}
&
\multicolumn{3}{c}{Route quality} \\
\cmidrule(lr){2-10}
Component
& Eval. \(n\)
& Accuracy
& Precision
& Recall
& F1
& AUC
& Routed \(n\)
& Purity
& Eff.\ acc.\\
\midrule
\shortstack[l]{Event vs.\ no event}
& 5,866
& 0.8607
& 0.6844
& 0.8555
& 0.7605
& 0.9306
& --
& --
& -- \\
\shortstack[l]{Weather vs.\ rest}
& 1,297
& 0.9530
& 1.0000
& 0.7240
& 0.8399
& 0.8306
& 1,895
& 0.6844
& 0.6522 \\
\shortstack[l]{Geopol.\ vs. supply/fin.}
& 1,076
& 0.5288
& 0.5181
& 0.6729
& 0.5854
& 0.5965
& 1,731
& 0.6216
& 0.3287 \\
\bottomrule
\end{tabular}
}
\end{table*}

Table~\ref{tab:multi_confusion} shows that windows without events are identified most reliably. Weather and geopolitical events remain distinguishable, whereas supply-financial events are often assigned to the geopolitical class. This pattern is consistent with the difficulty of the second Stage~2 classifier. On the final joint test set described in Section~\ref{sec:data_labelling}, the complete sequential classifier achieves an accuracy of $0.7639$ and a macro F1 score of $0.6067$.

\begin{table*}[htbp]
\centering
\caption{Confusion matrix normalised by row for the final classification into four classes. Each entry reports the percentage of windows within the corresponding true class.}
\label{tab:multi_confusion}
\setlength{\tabcolsep}{18pt}
\renewcommand{\arraystretch}{1.15}
{\small
\begin{tabular}{lcccc}
\toprule
True $\backslash$ Pred.\ (\%)
& No event
& Weather
& Geopolitical
& Supply-financial \\
\midrule
No event
& \textbf{86.3}
& 0.1
& 9.8
& 3.9 \\
Weather
& 10.5
& \textbf{64.8}
& 24.7
& 0.0 \\
Geopolitical
& 6.0
& 0.0
& \textbf{63.3}
& 30.7 \\
Supply-financial
& 22.6
& 0.0
& 47.4
& \textbf{30.0} \\
\bottomrule
\end{tabular}
}
\end{table*}

\begin{figure}[H]
    \centering
    \includegraphics[width=0.75\linewidth]{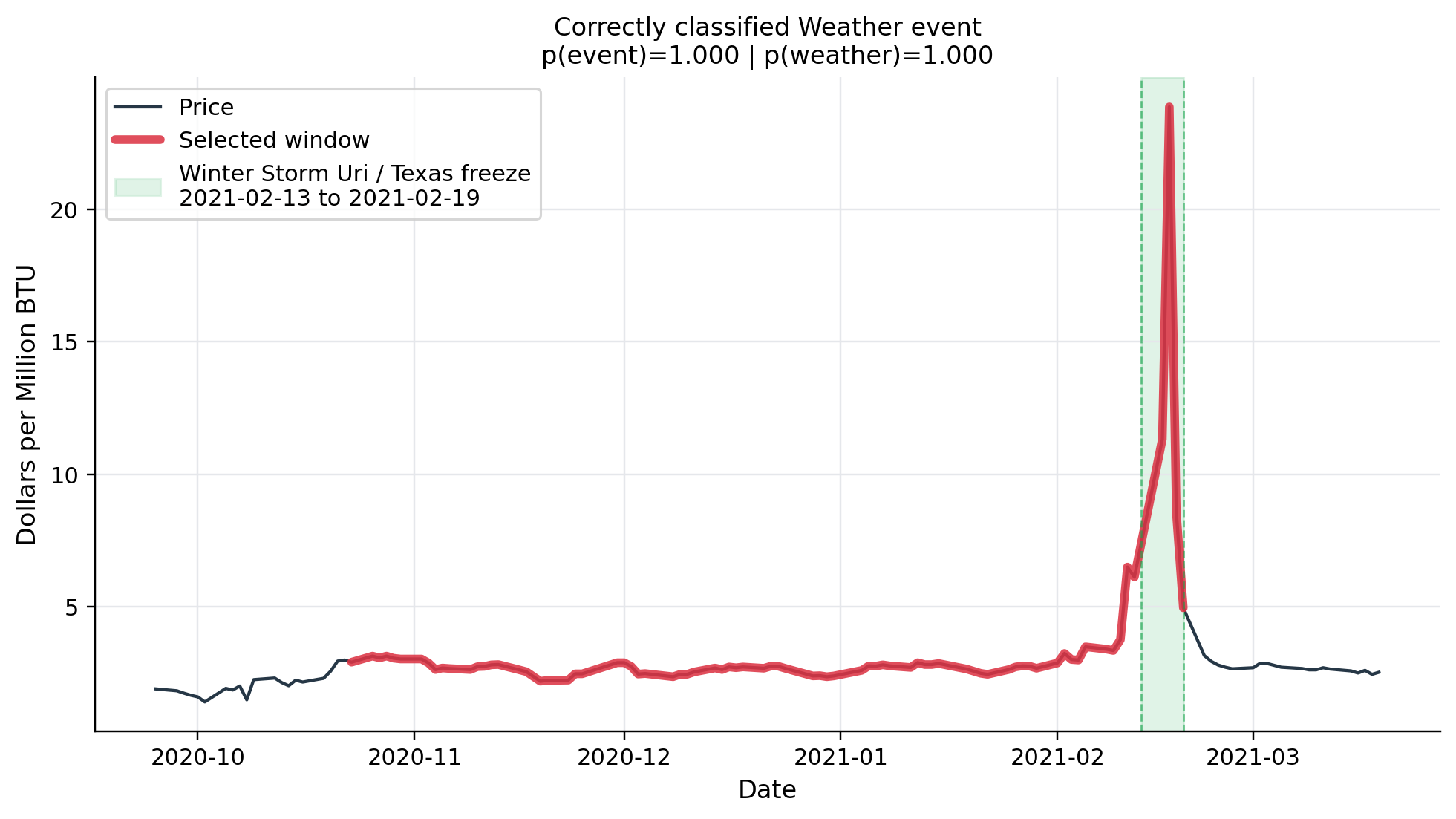}
    \caption{A correctly classified weather event.}
    \label{fig:multi-weather}
\end{figure}

Figure~\ref{fig:multi-weather} presents the correctly classified weather example. Winter Storm Uri produces a pronounced and brief price spike, giving the classifier a distinctive local pattern for identifying the weather family. Figure~\ref{fig:multi-nonweather} compares correct and incorrect classifications for the two other event families. The correctly classified supply-financial window follows a relatively smooth rise, whereas the supply-financial window assigned to the geopolitical class contains larger and more irregular fluctuations. Both errors between families have probabilities close to the decision threshold, indicating uncertainty in the assignment between these families. The correctly classified and misclassified geopolitical windows both concern Russia's invasion of Ukraine but cover different positions relative to the event. Their different predictions show that classification is influenced by the price pattern visible within a particular window.

\begin{figure}[ht]
    \centering
    \includegraphics[width=0.49\linewidth]{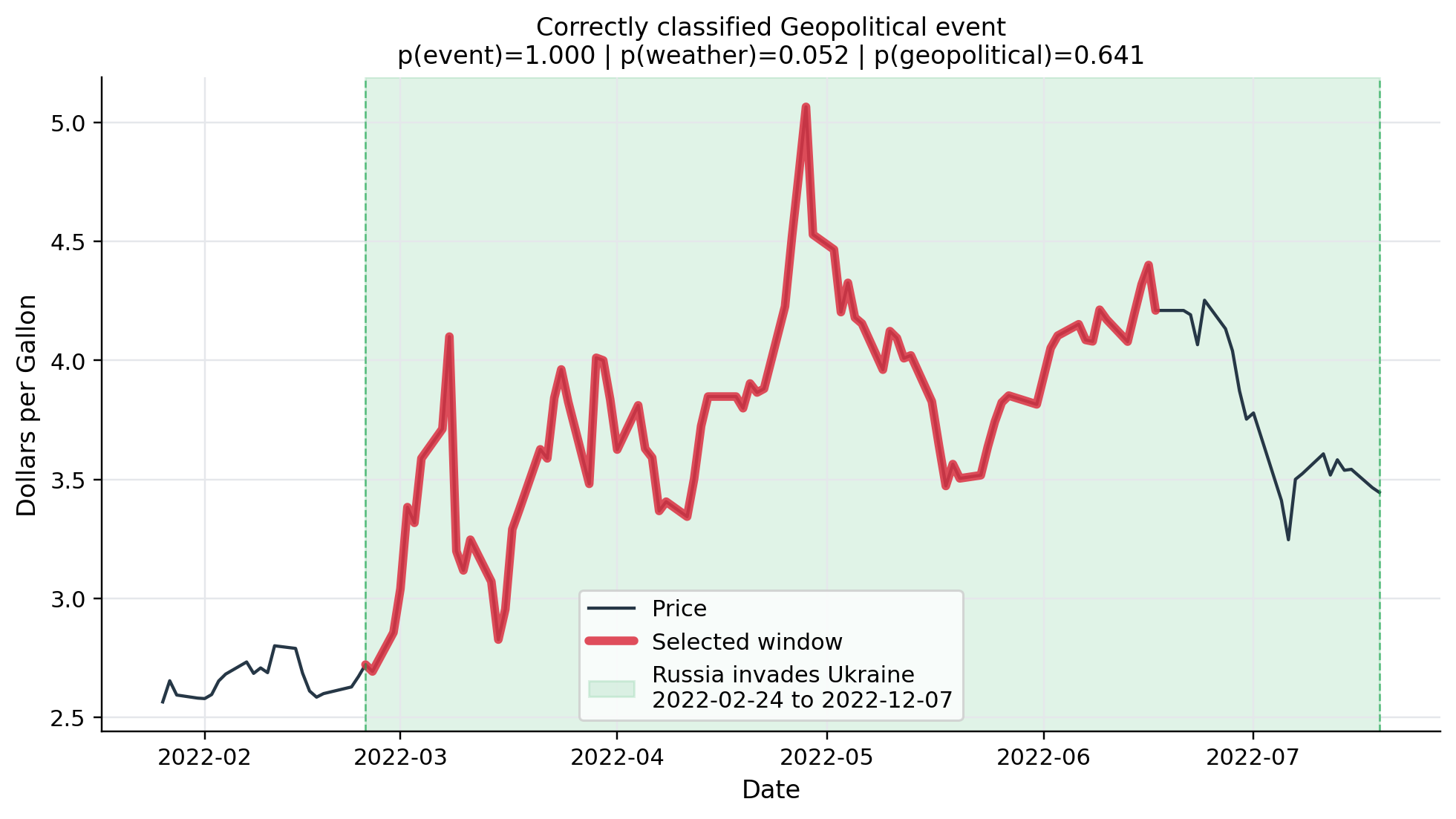}
    \includegraphics[width=0.49\linewidth]{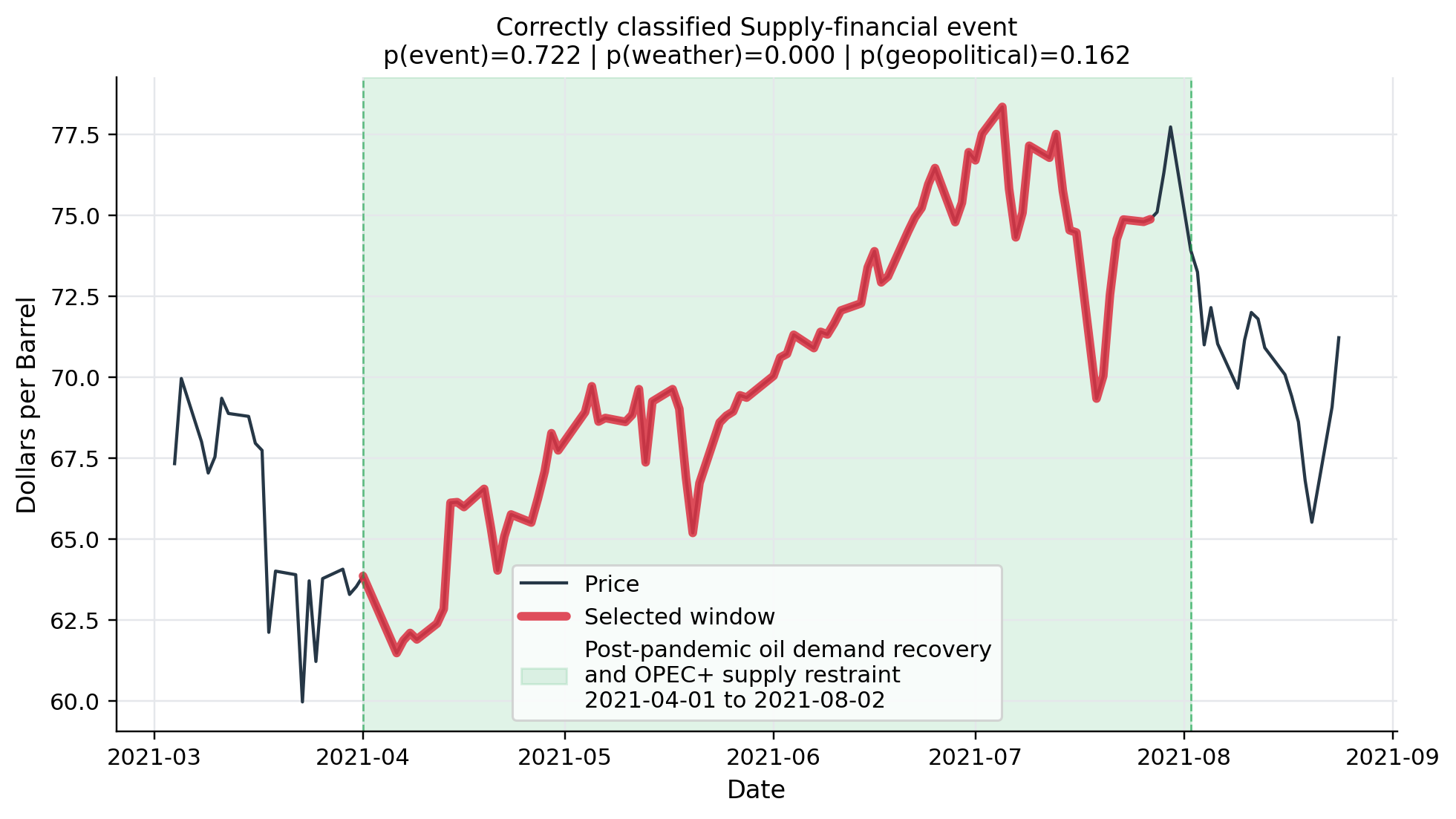}
    \includegraphics[width=0.49\linewidth]{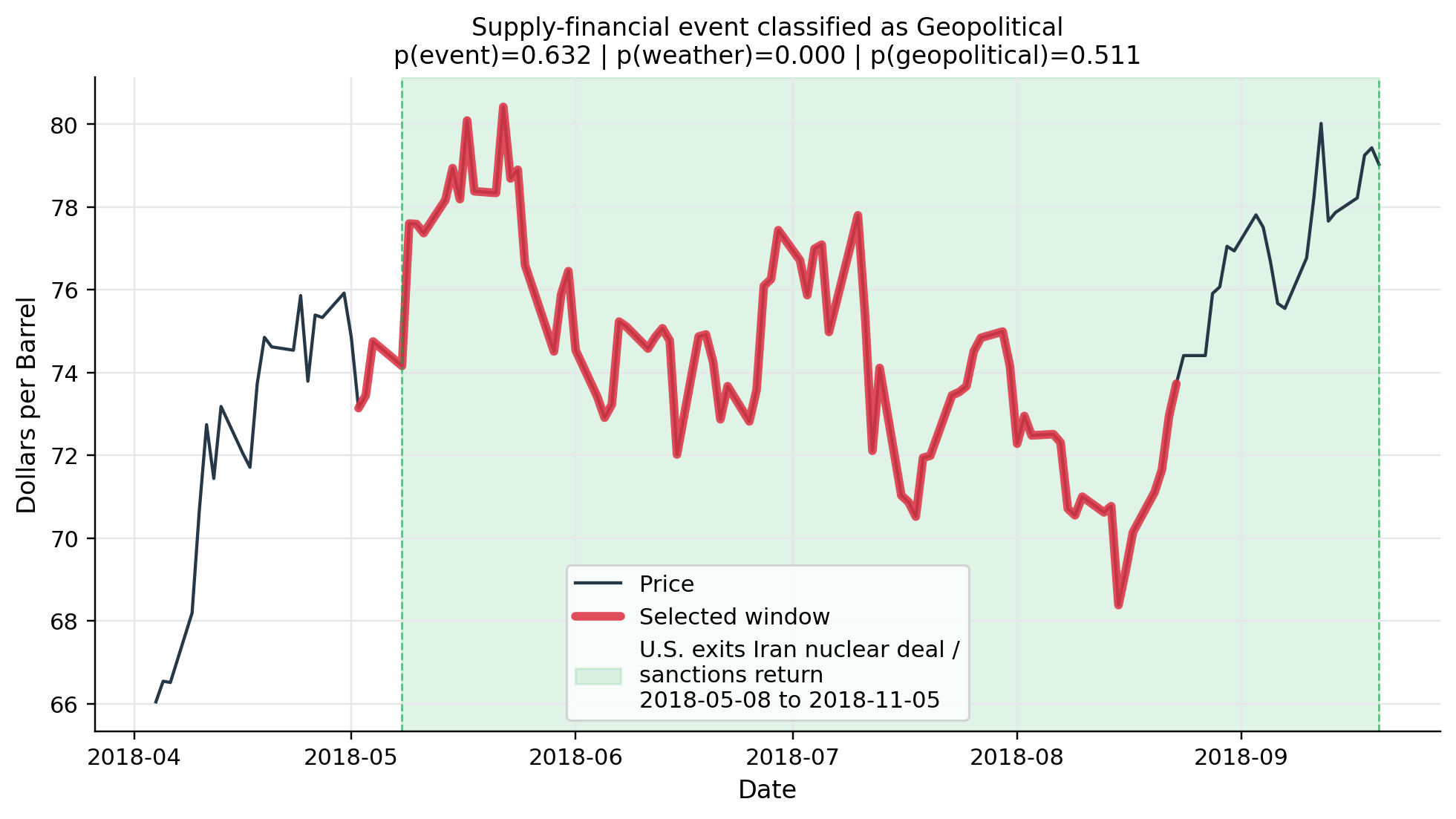}
    \includegraphics[width=0.49\linewidth]{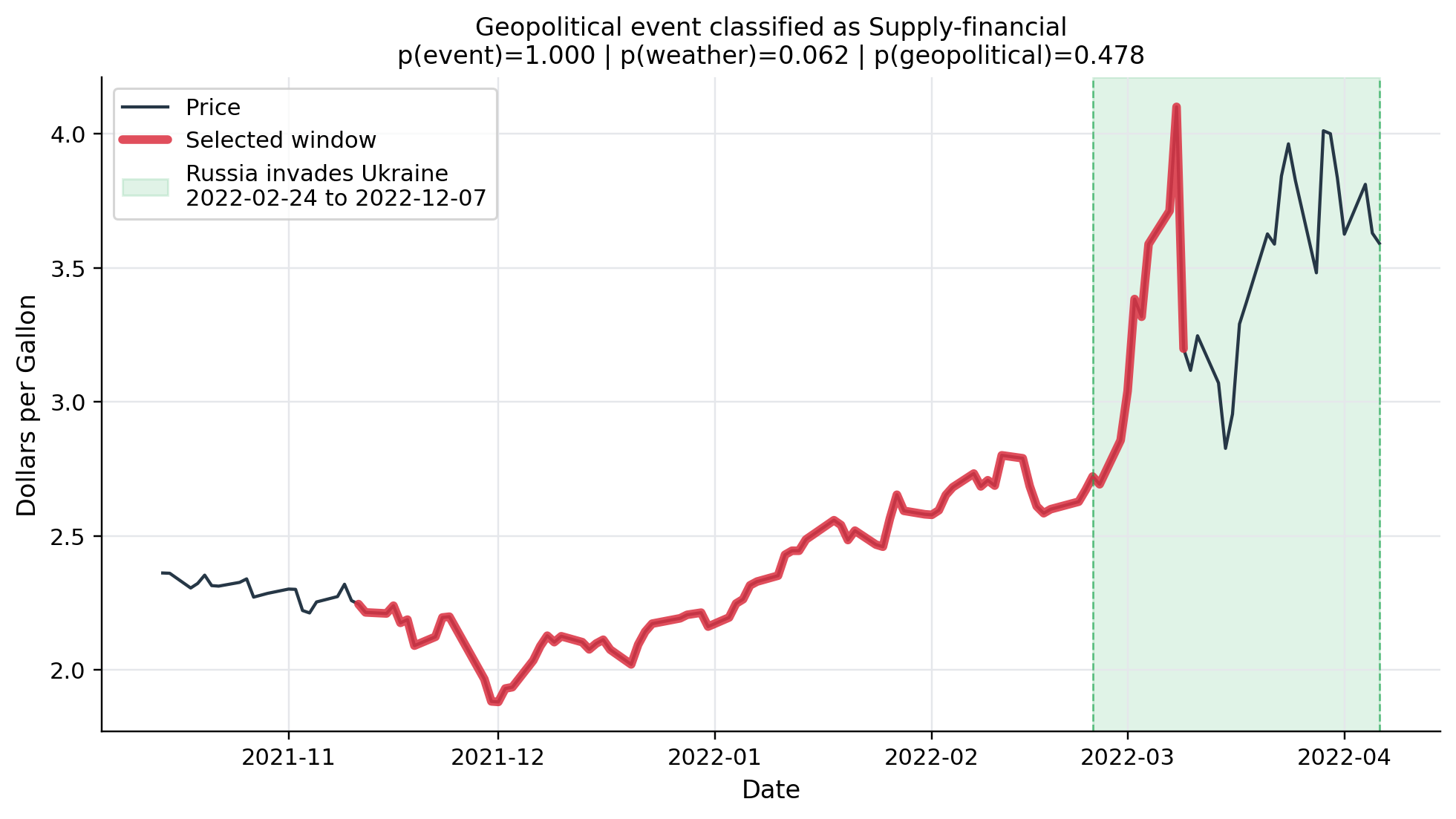}
    \caption{Representative geopolitical and supply-financial classifications. The upper panels are correctly classified geopolitical and supply-financial events. The lower panels show the corresponding errors between families.}
    \label{fig:multi-nonweather}
\end{figure}

\subsection{Case study}

We apply the fitted hierarchical classifier to the same six commodity series from 20 February to 20 July 2026. Each series contributes 103 withheld observations and 103 predictions at window endpoints based on rolling windows of 80 observations, giving 618 predictions in total. All observations were excluded from training, validation, selection of window length and threshold selection, and the models were not retrained for this analysis. The first window ending on 20 February uses the preceding 79 observations solely as lookback information.

\begin{figure}[H]
    \centering
    \includegraphics[width=0.88\linewidth]{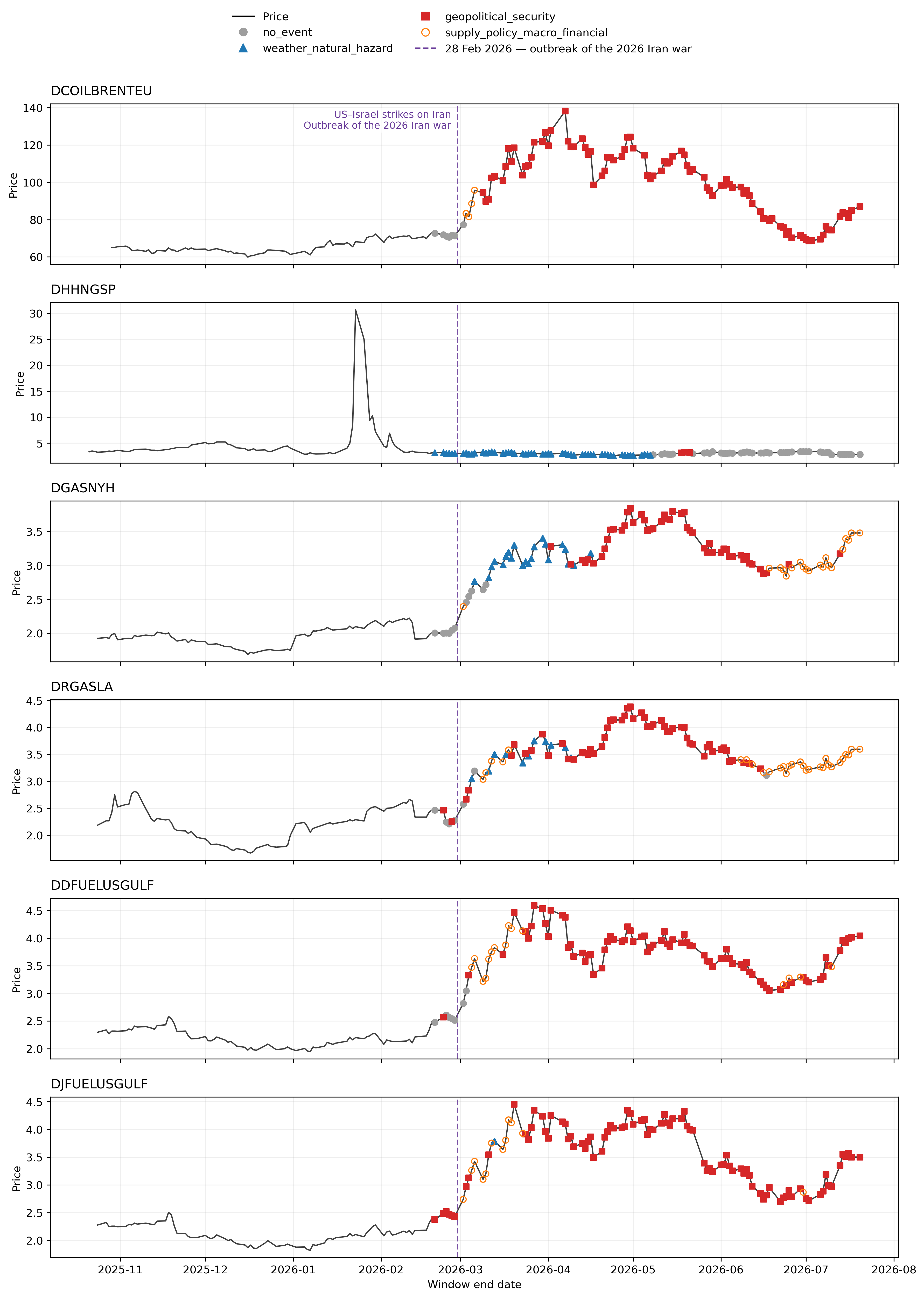}
    \caption{Classifications of independent observations for the six energy commodity price series around the outbreak of the 2026 Iran war. Each symbol gives the final hierarchical classification of the window of 80 observations ending on the plotted date. The purple dashed line marks 28 February 2026.}
    \label{fig:case-study-iran-war}
\end{figure}

Figure~\ref{fig:case-study-iran-war} plots the price series and the final hierarchical classification for every window. Each classification marker is placed on the final date of its corresponding window. The 79 prices preceding the first prediction date therefore provide the initial lookback context but have no classification markers. The dashed line marks the strikes on Iran by the United States and Israel on 28 February 2026 and the outbreak of the 2026 Iran war.

Brent crude oil (DCOILBRENTEU), the gasoline series for New York Harbor (DGASNYH) and Los Angeles (DRGASLA), and the Gulf Coast diesel (DDFUELUSGULF) and jet fuel (DJFUELUSGULF) series are classified as containing an event in most windows. Geopolitical classifications account for most predictions for Brent crude oil and the two Gulf Coast products. Their common transition towards this family after the marked date indicates that the corresponding price paths resemble the geopolitical patterns learned from the historical sample. Both gasoline series move between weather, geopolitical and supply-financial classifications, although to different degrees, indicating greater uncertainty about the event family as their price trajectories change.

Henry Hub natural gas (DHHNGSP) exhibits a different pattern. The spike coincides with Winter Storm Fern and the associated Arctic cold outbreak. An \href{https://www.eia.gov/pressroom/releases/press583.php}{EIA market update issued in February 2026} attributes the January surge in natural gas prices to increased heating demand, production reductions associated with weather and record storage withdrawals. The classifier accordingly assigns many early windows to the weather family, whose historical examples often contain pronounced brief spikes. As the rolling window advances and the spike leaves the input of 80 observations, the classification changes mainly to no event. Because no curated event labels after 20 February 2026 are used here, these results describe the fixed model's classifications on independent observations and are not estimates of classification accuracy or evidence that the war caused every observed price movement.

\section{Conclusion}\label{Conclusion}

This paper has developed a common CNN framework for detecting heterogeneous event-driven dynamics in time series windows. 
We showed that the induced CNN class exactly represents classifiers based on range, maximum drawup, maximum drawdown and slope change, and uniformly approximates classifiers based on realised volatility and autoregressive explosiveness on compact domains. Results for finite samples connect these constructions to error control for individual rules and to an oracle comparison over representative statistical classifiers. The simulation illustrates that the learned common head can retain the performance of a strong fixed rule and exploit complementary information across branches. In the empirical analysis, the hierarchical model classified events in six energy price series, while the independent 2026 case study distinguished a natural gas spike associated with weather from predominantly geopolitical patterns in several oil and refined product markets around the outbreak of the Iran war. These outputs are classifications rather than causal estimates.

Several extensions are natural. First, the binary representation and oracle theory could be developed for multiclass architectures with a shared feature extractor and a head that produces multiple scores. Second, the framework could be adapted from retrospective detection to forecasting, including prediction of future event occurrence, event type or subsequent price evolution under sequential updating. Third, multivariate time series and series observed on networks would allow the model to incorporate dependence across series and the propagation of shocks among connected markets or other interacting units. Graph convolutional branches could combine temporal features with network structure while retaining interpretable statistical comparators. Developing uncertainty quantification and learning guarantees for these extensions remains an important direction for future work.

\section*{Declarations}
\noindent
\textbf{Funding.} No funding was received.
\textbf{Competing interests.} The authors declare no competing interests.
\textbf{Generative AI use.} OpenAI Codex assisted with language editing, code development and review, and checks of mathematical arguments and numerical analyses. The authors reviewed and revised all AI-assisted material and take full responsibility for the manuscript.
\textbf{Data and code.} The energy price series used in this study are publicly available from the Federal Reserve Economic Data database. The curated event data, code, frozen data files and reported numerical outputs are available at \url{https://github.com/Xucaixia1211/Event-driven-energy-prices-by-CNN}.

\bibliographystyle{chicago}
\bibliography{reference}

\clearpage
\setcounter{section}{0}
\setcounter{equation}{0}
\setcounter{table}{0}
\setcounter{figure}{0}
\setcounter{theorem}{0}
\setcounter{lemma}{0}
\setcounter{cor}{0}
\setcounter{definition}{0}
\setcounter{remark}{0}
\renewcommand{\thesection}{S.\arabic{section}}
\renewcommand{\theequation}{S.\arabic{equation}}
\renewcommand{\thetable}{S.\arabic{table}}
\renewcommand{\thefigure}{S.\arabic{figure}}
\renewcommand{\thetheorem}{S.\arabic{theorem}}
\renewcommand{\thelemma}{S.\arabic{lemma}}
\renewcommand{\thecor}{S.\arabic{cor}}
\renewcommand{\thedefinition}{S.\arabic{definition}}
\renewcommand{\theremark}{S.\arabic{remark}}
\renewcommand{\theHsection}{S.\arabic{section}}
\renewcommand{\theHsubsection}{S.\arabic{section}.\arabic{subsection}}
\renewcommand{\theHequation}{S.\arabic{equation}}
\renewcommand{\theHtable}{S.\arabic{table}}
\renewcommand{\theHfigure}{S.\arabic{figure}}
\renewcommand{\theHtheorem}{S.\arabic{theorem}}
\renewcommand{\theHlemma}{S.\arabic{lemma}}
\renewcommand{\theHcor}{S.\arabic{cor}}
\renewcommand{\theHdefinition}{S.\arabic{definition}}
\renewcommand{\theHremark}{S.\arabic{remark}}
\setlength{\emergencystretch}{4em}\sloppy

\centerline{\large\bf Supplementary material for}
\vspace{4pt}
\centerline{\large\bf ``A convolutional framework for detecting}
\vspace{2pt}
\centerline{\large\bf event-driven dynamics in energy price series''}
\vspace{.4cm}
\vspace{.6cm}

\noindent
This supplement contains proofs of Theorems 1--9 of the main paper in Section S.1, followed by representation and approximation results for alternative architectures in Section S.2 and additional numerical checks in Section S.3. Numbers prefixed by S belong to this supplement. Unprefixed numbers refer to the main paper. Throughout, \(\sigma(x)=\max\{x,0\}\) and \(\mathbb I\{\cdot\}\) denotes the indicator function.

\section{ Proofs of the results in the main paper}

This section provides the proofs of the results stated in the main paper.

\subsection{Proof of Theorem 1 (range classifier)}

\begin{proof}
Set \(B=1\), \(L_1=1\), \(k_{1,1}=1\), \(C_{1,0}=1\), \(C_{1,1}=2\), \(m_1=2\), and \(\mathrm{Pool}_{1,1}=\operatorname{Id}\). Choose $w_{1,1,0}^{(1,1)}=1$, $\beta_1^{(1,1)}=0$, $w_{2,1,0}^{(1,1)}=-1$, $\beta_2^{(1,1)}=0$. Then $\widetilde Z_{1,t}^{(1,1)}=\sigma(X_t)$, $\widetilde Z_{2,t}^{(1,1)}=\sigma(-X_t)$. Since \(\mathrm{Pool}_{1,1}=\operatorname{Id}\),
\(Z_{j,t}^{(1,1)}=\widetilde Z_{j,t}^{(1,1)}\).
Define $(a_{1,0}^{(1)},a_{1,1}^{(1)},a_{1,2}^{(1)})=(0,1,-1)$, $(a_{2,0}^{(1)},a_{2,1}^{(1)},a_{2,2}^{(1)})=(0,-1,1)$. Using \(\sigma(x)-\sigma(-x)=x\), $G_{1,t}^{(1)}=X_t, G_{2,t}^{(1)}=-X_t.$ Global max pooling gives $P_1^{(1),\max}=\max_{1\leq t\leq T}X_t, P_2^{(1),\max}=-\min_{1\leq t\leq T}X_t.$ Within the affine submodel $g(\bm P)=\gamma_0+\bm\gamma^\top\bm P$ contained in the fixed head, choose $\gamma_0=-\lambda, \gamma_1^{(1),\max}=\gamma_2^{(1),\max}=1,$ and set all coefficients for sum pooling to zero. Then
\[
f_{R,\lambda}(\bm X)
=\max_tX_t-\min_tX_t-\lambda
=\mathcal R(\bm X)-\lambda.
\]
Thus $\mathbb I\{f_{R,\lambda}(\bm X)>0\}
=\mathbb I\{\mathcal R(\bm X)>\lambda\}.$
\end{proof}

\subsection{Proof of Theorem 2 (directional classifier)}

\begin{proof}
\begin{itemize}
    \item [(1)] Representation of the maximum drawup classifier.

For each lag \(h=1,\ldots,T-1\), introduce one branch. Thus \(B=T-1\). Branch \(h\) compares \(X_t\) with \(X_{t+h}\). Set $L_h=1$, $k_{h,1}=h+1$, $C_{h,0}=1$, $C_{h,1}=1$, $m_h=1$, and $\mathrm{Pool}_{h,1}=\mathrm{Id}$. For branch \(h\), set $w^{(h,1)}_{1,1,0}=-1$, $w^{(h,1)}_{1,1,h}=1$, and $w^{(h,1)}_{1,1,r}=0$ for $r=1,\ldots,h-1$. Set $\beta^{(h,1)}_1=-\lambda$. Then, for $t=1,\ldots,T-h$,
\[
\begin{aligned}
\widetilde Z^{(h,1)}_{1,t}
&=
\sigma\left(
-\lambda
+
\sum_{r=0}^{h}
w^{(h,1)}_{1,1,r}X_{t+r}
\right)\\
&=
\sigma(X_{t+h}-X_t-\lambda).
\end{aligned}
\]
Since \(\mathrm{Pool}_{h,1}=\mathrm{Id}\), $Z^{(h,1)}_{1,t}
=
\widetilde Z^{(h,1)}_{1,t}$.

Take the local output layer to be the identity by setting $a^{(h)}_{1,0}=0$, $a^{(h)}_{1,1}=1$. Then $G^{(h)}_{1,t}
=
Z^{(h,1)}_{1,t}
=
\sigma(X_{t+h}-X_t-\lambda)$. By global max pooling,
\[
P^{(h),\max}_1
=
\max_{1\le t\le T-h}
G^{(h)}_{1,t}
=
\max_{1\le t\le T-h}
\sigma(X_{t+h}-X_t-\lambda).
\]

Within the affine submodel $g(\bm P)=\gamma_0+\bm\gamma^\top\bm P$ contained in the fixed head, choose $\gamma_0=0$, $\gamma^{(h),\max}_1=1$, $\gamma^{(h),\mathrm{sum}}_1=0$, $h=1,\ldots,T-1$. Then the resulting score function is
\[
f_{D^+,\lambda}(X)
=
\sum_{h=1}^{T-1}
P^{(h),\max}_1
=
\sum_{h=1}^{T-1}
\max_{1\le t\le T-h}
\sigma(X_{t+h}-X_t-\lambda).
\]
Since all terms in the sum are nonnegative, $f_{D^+,\lambda}(X)>0$ if and only if there exist \(h\in\{1,\ldots,T-1\}\) and \(t\in\{1,\ldots,T-h\}\) such that $X_{t+h}-X_t>\lambda$. This is equivalent to $D^+(X)>\lambda$. Therefore,
\[
\mathbb I\{f_{D^+,\lambda}(X)>0\}
=
\mathbb I\{D^+(X)>\lambda\}
=
h_{D^+,\lambda}(X).
\]
This completes the proof.
\item [(2)] Representation of the maximum drawdown classifier
For each lag \(h=1,\ldots,T-1\), introduce one branch. Thus \(B=T-1\). Set $L_h=1$, $k_{h,1}=h+1$, $C_{h,0}=1$, $C_{h,1}=1$, $m_h=1$, and $\mathrm{Pool}_{h,1}=\mathrm{Id}$. For branch \(h\), set $w^{(h,1)}_{1,1,0}=1$, $w^{(h,1)}_{1,1,h}=-1$, and $w^{(h,1)}_{1,1,r}=0$, $r=1,\ldots,h-1$. Set $\beta^{(h,1)}_1=-\lambda$. Then, for \(t=1,\ldots,T-h\),
\[
\begin{aligned}
\widetilde Z^{(h,1)}_{1,t}
&=
\sigma\left(
-\lambda
+
\sum_{r=0}^{h}
w^{(h,1)}_{1,1,r}X_{t+r}
\right)\\
&=
\sigma(X_t-X_{t+h}-\lambda).
\end{aligned}
\]
Since \(\mathrm{Pool}_{h,1}=\mathrm{Id}\), $Z^{(h,1)}_{1,t} = \widetilde Z^{(h,1)}_{1,t}$.

Take the local output layer to be the identity by setting $a^{(h)}_{1,0}=0$, $a^{(h)}_{1,1}=1$. Then $G^{(h)}_{1,t}
=
Z^{(h,1)}_{1,t}
=
\sigma(X_t-X_{t+h}-\lambda)$. By global max pooling,
\[
P^{(h),\max}_1
=
\max_{1\le t\le T-h}
G^{(h)}_{1,t}
=
\max_{1\le t\le T-h}
\sigma(X_t-X_{t+h}-\lambda).
\]

Within the same affine submodel, choose $\gamma_0=0$, $\gamma^{(h),\max}_1=1$, $\gamma^{(h),\mathrm{sum}}_1=0$, $h=1,\ldots,T-1$. Then the resulting score function is
\[
f_{D^-,\lambda}(X)
=
\sum_{h=1}^{T-1}
P^{(h),\max}_1
=
\sum_{h=1}^{T-1}
\max_{1\le t\le T-h}
\sigma(X_t-X_{t+h}-\lambda).
\]
Since all terms in the sum are nonnegative, $f_{D^-,\lambda}(X)>0$ if and only if there exist \(h\in\{1,\ldots,T-1\}\) and \(t\in\{1,\ldots,T-h\}\) such that $X_t-X_{t+h}>\lambda$. This is equivalent to $D^-(X)>\lambda$. Therefore,
\[
\mathbb I\{f_{D^-,\lambda}(X)>0\}
=
\mathbb I\{D^-(X)>\lambda\}
=
h_{D^-,\lambda}(X).
\]
This completes the proof.
\end{itemize}
\end{proof}


\subsection{Proof of Theorem 3 (slope change classifier)}

\begin{proof}
For each \(\tau\in\mathcal I\), the slope contrast \(C_\tau(X)\) is a linear function of \(X\). Indeed, $C_\tau(X) = \sum_{t=1}^{T}a_{\tau,t}X_t$, where
\[
a_{\tau,t}
=
\begin{cases}
\dfrac{t-\bar t_L(\tau)}{D_L(\tau)},
& 1\le t\le \tau,\\[1.2em]
-\dfrac{t-\bar t_R(\tau)}{D_R(\tau)},
& \tau+1\le t\le T.
\end{cases}
\]

For each \(\tau\in\mathcal I\), introduce two branches, indexed by \((\tau,+)\) and \((\tau,-)\). Thus $B=2|\mathcal I|=2(T-4)$. Each branch has one convolutional block, one output channel, a kernel spanning the full window with size \(T\), and no local pooling:
\[
L_{\tau,+}=L_{\tau,-}=1,
\]
\[
C_{\tau,+,0}=C_{\tau,-,0}=1,
\qquad
C_{\tau,+,1}=C_{\tau,-,1}=1,
\]
\[
k_{\tau,+,1}=k_{\tau,-,1}=T,
\qquad
\mathrm{Pool}_{\tau,+,1}
=
\mathrm{Pool}_{\tau,-,1}
=
\mathrm{Id}.
\]
Since the kernel size is \(T\), each branch has only one valid time position.

For the positive branch \((\tau,+)\), set $w^{(\tau,+,1)}_{1,1,r} = a_{\tau,r+1}$, $r=0,\ldots,T-1$, and $\beta^{(\tau,+,1)}_1=-\lambda$. Then
\[
\begin{aligned}
\widetilde Z^{(\tau,+,1)}_{1,1}
&=
\sigma\left(
-\lambda
+
\sum_{r=0}^{T-1}
a_{\tau,r+1}X_{1+r}
\right)\\
&=
\sigma(C_\tau(X)-\lambda).
\end{aligned}
\]

For the negative branch \((\tau,-)\), set $w^{(\tau,-,1)}_{1,1,r}
=
-a_{\tau,r+1}$, $r=0,\ldots,T-1$ and $\beta^{(\tau,-,1)}_1=-\lambda$. Then
\[
\begin{aligned}
\widetilde Z^{(\tau,-,1)}_{1,1}
&=
\sigma\left(
-\lambda
-
\sum_{r=0}^{T-1}
a_{\tau,r+1}X_{1+r}
\right)\\
&=
\sigma(-C_\tau(X)-\lambda).
\end{aligned}
\]

For each branch, take the local output layer to be the identity, which is $a^{(\tau,+)}_{1,0}=0$, $a^{(\tau,+)}_{1,1}=1$, and $a^{(\tau,-)}_{1,0}=0$, $a^{(\tau,-)}_{1,1}=1$. Thus $G^{(\tau,+)}_{1,1}
=
\sigma(C_\tau(X)-\lambda)$, and $G^{(\tau,-)}_{1,1}
=
\sigma(-C_\tau(X)-\lambda)$. Since each branch has only one time position, global max pooling gives
\[
P^{(\tau,+),\max}_1
=
G^{(\tau,+)}_{1,1},
\qquad
P^{(\tau,-),\max}_1
=
G^{(\tau,-)}_{1,1}.
\]

Within the affine submodel $g(\bm P)=\gamma_0+\bm\gamma^\top\bm P$ contained in the fixed head, choose $\gamma_0=0$, $\gamma^{(\tau,+),\max}_1=1$, $\gamma^{(\tau,-),\max}_1=1$ for all \(\tau\in\mathcal I\), and set all coefficients for sum pooling to zero. Then
\[
\begin{aligned}
f_{S,\lambda}(X)
&=
\sum_{\tau\in\mathcal I}
P^{(\tau,+),\max}_1
+
\sum_{\tau\in\mathcal I}
P^{(\tau,-),\max}_1\\
&=
\sum_{\tau\in\mathcal I}
\sigma(C_\tau(X)-\lambda)
+
\sum_{\tau\in\mathcal I}
\sigma(-C_\tau(X)-\lambda).
\end{aligned}
\]
All terms are nonnegative. Therefore, $f_{S,\lambda}(X)>0$ if and only if there exists \(\tau\in\mathcal I\) such that $C_\tau(X)>\lambda$ or $-C_\tau(X)>\lambda$. Equivalently, $|C_\tau(X)|>\lambda$ for some \(\tau\in\mathcal I\), which is exactly $S(X)>\lambda$. Hence
\[
\mathbb I\{f_{S,\lambda}(X)>0\}
=
\mathbb I\{S(X)>\lambda\}
=
h_{S,\lambda}(X).
\]
This completes the proof.
\end{proof}


\subsection{Proof of Theorem 4 (realised volatility)}
\begin{proof}
Since \(\bm X\in[-M,M]^T\), for every \(t=1,\ldots,T-1\), \(r_t=X_{t+1}-X_t\in[-2M,2M]\).
Let $\rho=\frac{\varepsilon}{T-1}$. By the ReLU approximation result for the square function on a bounded interval \citep{yarotsky2017error}, there exists a ReLU network with one input \(\mathrm{Sq}_{\rho,M}\) such that
\[
\sup_{u\in[-2M,2M]}
\left|
\mathrm{Sq}_{\rho,M}(u)-u^2
\right|
\le
\rho.
\]
After scaling the construction that approximates the square function to \([-2M,2M]\), its depth and number of parameters are \(O\{\log(2+(T-1)M^2/\varepsilon)\}\).

Write this ReLU network as follows. The first hidden layer is $H^{(1)}_j(u)
=
\sigma\left(
\alpha^{(1)}_j u+\beta^{(1)}_j
\right)$, $j=1,\ldots,C_1$. For \(\ell=2,\ldots,L\), the hidden layers are
\[
H^{(\ell)}_j(u)
=
\sigma\left(
\beta^{(\ell)}_j
+
\sum_{i=1}^{C_{\ell-1}}
A^{(\ell)}_{ji}H^{(\ell-1)}_i(u)
\right),
\qquad
j=1,\ldots,C_\ell.
\]
The output layer is $\mathrm{Sq}_{\rho,M}(u)
=
a_0+
\sum_{j=1}^{C_L}a_jH^{(L)}_j(u)$.

We now construct a CNN. Take \(B=1\), \(L_1=L\), and set $k_{1,1}=2$, $k_{1,\ell}=1$, $\ell=2,\ldots,L$. Let $C_{1,0}=1$, $C_{1,\ell}=C_\ell$, $\ell=1,\ldots,L$, and set $\mathrm{Pool}_{1,\ell}=\mathrm{Id}$, $\ell=1,\ldots,L$.

For the first CNN block, set $w^{(1,1)}_{j,1,0}
=
-\alpha^{(1)}_j$, $w^{(1,1)}_{j,1,1}
=
\alpha^{(1)}_j$, $\beta^{(1,1)}_j
=
\beta^{(1)}_j$. Then, for \(t=1,\ldots,T-1\),
\[
\begin{aligned}
\widetilde Z^{(1,1)}_{j,t}
&=
\sigma\left(
\beta^{(1)}_j
-\alpha^{(1)}_jX_t
+
\alpha^{(1)}_jX_{t+1}
\right)\\
&=
\sigma\left(
\beta^{(1)}_j+\alpha^{(1)}_jr_t
\right)\\
&=
H^{(1)}_j(r_t).
\end{aligned}
\]
Since \(\mathrm{Pool}_{1,1}=\mathrm{Id}\), we have $Z^{(1,1)}_{j,t}=H^{(1)}_j(r_t)$. For each subsequent block \(\ell=2,\ldots,L\), set $w^{(1,\ell)}_{j,i,0}
=
A^{(\ell)}_{ji}$, $\beta^{(1,\ell)}_j
=
\beta^{(\ell)}_j$. Then, by induction over \(\ell\), $Z^{(1,\ell)}_{j,t}
=
H^{(\ell)}_j(r_t)$ for $j=1,\ldots,C_\ell$, $t=1,\ldots,T-1$.

Define one local output channel by taking \(m_1=1\) and setting $a^{(1)}_{1,0}=a_0$, $a^{(1)}_{1,j}=a_j$, $j=1,\ldots,C_L$. Then
\[
\begin{aligned}
G^{(1)}_{1,t}
&=
a_0+
\sum_{j=1}^{C_L}a_jZ^{(1,L)}_{j,t}\\
&=
a_0+
\sum_{j=1}^{C_L}a_jH^{(L)}_j(r_t)\\
&=
\mathrm{Sq}_{\rho,M}(r_t).
\end{aligned}
\]

Within the affine submodel $g(\bm P)=\gamma_0+\bm\gamma^\top\bm P$ contained in the fixed head, use global sum pooling by choosing $\gamma^{(1),\mathrm{sum}}_1=1$, $\gamma^{(1),\max}_1=0$, and $\gamma_0=0$. Then the resulting CNN score function is
\[
f^{\mathrm{RV}}_{\varepsilon,M}(\bm X)
=
P^{(1),\mathrm{sum}}_1
=
\sum_{t=1}^{T-1}G^{(1)}_{1,t}
=
\sum_{t=1}^{T-1}\mathrm{Sq}_{\rho,M}(r_t).
\]
Hence, for every \(\bm X\in[-M,M]^T\),
\[
\begin{aligned}
\left|
f^{\mathrm{RV}}_{\varepsilon,M}(\bm X)-V(\bm X)
\right|
&=
\left|
\sum_{t=1}^{T-1}
\left[
\mathrm{Sq}_{\rho,M}(r_t)-r_t^2
\right]
\right|\\
&\le
\sum_{t=1}^{T-1}
\left|
\mathrm{Sq}_{\rho,M}(r_t)-r_t^2
\right|\\
&\le
(T-1)\rho=
\varepsilon.
\end{aligned}
\]
Taking the supremum over \(\bm X\in[-M,M]^T\) gives
\[
\sup_{\bm X\in[-M,M]^T}
\left|
f^{\mathrm{RV}}_{\varepsilon,M}(\bm X)-V(\bm X)
\right|
\le
\varepsilon.
\]
This completes the proof.
\end{proof}

\subsection{Proof of Theorem 5 (autoregressive statistic)}

\begin{proof}
Define the local function $g(a,b)=a(b-a)$. Then
\[
S_{\mathrm{AR}}(\bm X)
=
\sum_{t=1}^{T-1}g(X_t,X_{t+1})
=
\sum_{t=1}^{T-1}X_t(X_{t+1}-X_t).
\]
Since \(\bm X\in[-M,M]^T\), for every \(t=1,\ldots,T-1\), \(X_t\in[-M,M]\) and \(X_{t+1}-X_t\in[-2M,2M]\).
Let \(\rho=\varepsilon/(T-1)\). By the ReLU approximation result for multiplication on a bounded domain \citep{yarotsky2017error}, there exists a ReLU network with two inputs \(\mathrm{Mult}_{\rho,M}\) such that
\[
\sup_{|u|\le M,\ |v|\le 2M}
\left|
\mathrm{Mult}_{\rho,M}(u,v)-uv
\right|
\le
\rho.
\]
After scaling the construction that approximates multiplication to this domain, its depth and number of parameters are \(O\{\log(2+(T-1)M^2/\varepsilon)\}\).

Write this ReLU network as follows. The first hidden layer is
\[
H^{(1)}_j(u,v)
=
\sigma\left(
\alpha^{(1)}_{j,1}u
+
\alpha^{(1)}_{j,2}v
+
\beta^{(1)}_j
\right),
\qquad
j=1,\ldots,C_1.
\]
For \(\ell=2,\ldots,L\), the hidden layers are
\[
H^{(\ell)}_j(u,v)
=
\sigma\left(
\beta^{(\ell)}_j
+
\sum_{i=1}^{C_{\ell-1}}
A^{(\ell)}_{ji}H^{(\ell-1)}_i(u,v)
\right),
\qquad
j=1,\ldots,C_\ell.
\]
The output layer is
\[
\mathrm{Mult}_{\rho,M}(u,v)
=
a_0+
\sum_{j=1}^{C_L}a_jH^{(L)}_j(u,v).
\]

We now construct a CNN. Take \(B=1\), \(L_1=L\), and set $k_{1,1}=2$, $k_{1,\ell}=1$, $\ell=2,\ldots,L$. Let $C_{1,0}=1$, $C_{1,\ell}=C_\ell$, $\ell=1,\ldots,L$, and set $\mathrm{Pool}_{1,\ell}=\mathrm{Id}$ for $\ell=1,\ldots,L$.

At time \(t\), the first convolutional block observes \((X_t,X_{t+1})\). We apply the multiplication network to $u=X_t$, $v=X_{t+1}-X_t$. For each unit in the first layer, \(j=1,\ldots,C_1\),
\[
\begin{aligned}
H^{(1)}_j(X_t,X_{t+1}-X_t)
&=
\sigma\left(
\alpha^{(1)}_{j,1}X_t
+
\alpha^{(1)}_{j,2}(X_{t+1}-X_t)
+
\beta^{(1)}_j
\right)\\
&=
\sigma\left(
(\alpha^{(1)}_{j,1}-\alpha^{(1)}_{j,2})X_t
+
\alpha^{(1)}_{j,2}X_{t+1}
+
\beta^{(1)}_j
\right).
\end{aligned}
\]
Therefore, set $w^{(1,1)}_{j,1,0}
=
\alpha^{(1)}_{j,1}-\alpha^{(1)}_{j,2}$, $w^{(1,1)}_{j,1,1}
=
\alpha^{(1)}_{j,2}$, and $\beta^{(1,1)}_j
=
\beta^{(1)}_j$.
Then, for \(t=1,\ldots,T-1\), $\widetilde Z^{(1,1)}_{j,t}
=
H^{(1)}_j(X_t,X_{t+1}-X_t)$. Since \(\mathrm{Pool}_{1,1}=\mathrm{Id}\), $Z^{(1,1)}_{j,t}
=
H^{(1)}_j(X_t,X_{t+1}-X_t)$.

For each subsequent block \(\ell=2,\ldots,L\), set $w^{(1,\ell)}_{j,i,0}
=
A^{(\ell)}_{ji}$, $\beta^{(1,\ell)}_j
=
\beta^{(\ell)}_j$. Then, by induction over \(\ell\),
\[
Z^{(1,\ell)}_{j,t}
=
H^{(\ell)}_j(X_t,X_{t+1}-X_t),
\qquad
j=1,\ldots,C_\ell,\quad t=1,\ldots,T-1.
\]

Define one local output channel by taking \(m_1=1\) and setting $a^{(1)}_{1,0}=a_0$, $a^{(1)}_{1,j}=a_j$ for $j=1,\ldots,C_L$. Then
\[
\begin{aligned}
G^{(1)}_{1,t}
&=
a_0+
\sum_{j=1}^{C_L}a_jZ^{(1,L)}_{j,t}\\
&=
a_0+
\sum_{j=1}^{C_L}a_jH^{(L)}_j(X_t,X_{t+1}-X_t)\\
&=
\mathrm{Mult}_{\rho,M}(X_t,X_{t+1}-X_t).
\end{aligned}
\]

Within the affine submodel $g(\bm P)=\gamma_0+\bm\gamma^\top\bm P$ contained in the fixed head, use global sum pooling by choosing $\gamma^{(1),\mathrm{sum}}_1=1$, $\gamma^{(1),\max}_1=0$, and $\gamma_0=0$. Then the resulting CNN score function is
\[
f^{\mathrm{AR}}_{\varepsilon,M}(\bm X)
=
P^{(1),\mathrm{sum}}_1
=
\sum_{t=1}^{T-1}
\mathrm{Mult}_{\rho,M}(X_t,X_{t+1}-X_t).
\]
Hence, for every \(\bm X\in[-M,M]^T\),
\[
\begin{aligned}
\left|
f^{\mathrm{AR}}_{\varepsilon,M}(\bm X)-S_{\mathrm{AR}}(\bm X)
\right|
&=
\left|
\sum_{t=1}^{T-1}
\left[
\mathrm{Mult}_{\rho,M}(X_t,X_{t+1}-X_t)
-
X_t(X_{t+1}-X_t)
\right]
\right|\\
&\le
\sum_{t=1}^{T-1}
\left|
\mathrm{Mult}_{\rho,M}(X_t,X_{t+1}-X_t)
-
X_t(X_{t+1}-X_t)
\right|\\
&\le
(T-1)\rho =
\varepsilon.
\end{aligned}
\]
Taking the supremum over \(\bm X\in[-M,M]^T\) gives
\[
\sup_{\bm X\in[-M,M]^T}
\left|
f^{\mathrm{AR}}_{\varepsilon,M}(\bm X)-S_{\mathrm{AR}}(\bm X)
\right|
\le
\varepsilon.
\]
This completes the proof.
\end{proof}

\subsection{Proof of Theorem 6 (size and power of the slope)}
\begin{proof}
Under \(Y=0\), the least squares slope fitted to the left segment is
\[
\widehat b_L(\tau)
=
b+
\frac{
\sum_{t=1}^{\tau}
(t-\overline t_L(\tau))\varepsilon_t
}{
D_L(\tau)
},
\]
while the slope fitted to the right segment is
\[
\widehat b_R(\tau)
=
b+
\frac{
\sum_{t=\tau+1}^{T}
(t-\overline t_R(\tau))\varepsilon_t
}{
D_R(\tau)
}.
\]

It follows that \(C_\tau(\bm X)\) is Gaussian with mean zero.
Since the two sums involve disjoint sets of innovations,
its variance is
\[
\begin{aligned}
\operatorname{Var}
\left\{
C_\tau(\bm X)
\mid Y=0
\right\}
&=
\sigma_S^2
\left\{
\frac{
\sum_{t=1}^{\tau}
(t-\overline t_L(\tau))^2
}{
D_L^2(\tau)
}
+
\frac{
\sum_{t=\tau+1}^{T}
(t-\overline t_R(\tau))^2
}{
D_R^2(\tau)
}
\right\}
\\
&=
\sigma_S^2
\left\{
\frac{1}{D_L(\tau)}
+
\frac{1}{D_R(\tau)}
\right\}
\\
&=
\sigma_S^2v_\tau.
\end{aligned}
\]

Therefore, for every \(\lambda>0\),
\[
P\left\{
|C_\tau(\bm X)|>\lambda
\mid Y=0
\right\}
\leq
2
\exp\left\{
-\frac{\lambda^2}{2\sigma_S^2v_\tau}
\right\}.
\]

Using the union bound over the \(m_T\) candidate split points gives
\[
\begin{aligned}
P\left\{
S(\bm X)>\lambda
\mid Y=0
\right\}
&\leq
\sum_{\tau\in\mathcal I}
P\left\{
|C_\tau(\bm X)|>\lambda
\mid Y=0
\right\}
\\
&\leq
2m_T
\exp\left\{
-\frac{\lambda^2}{2\sigma_S^2v_T^{\max}}
\right\}.
\end{aligned}
\]

Substituting
\(\lambda=\lambda_{S,\alpha_S}\)
shows that
\[
P\left\{
h_{S,\lambda_{S,\alpha_S}}(\bm X)=1
\mid Y=0
\right\}
\leq
\alpha_S.
\]

Under \(Y=1\), evaluate the contrast at the true split point
\(\tau^\star\).
The left and right least squares slopes are unbiased for
\(b_L\) and \(b_R\), respectively, and hence
\[
C_{\tau^\star}(\bm X)
\sim
N\left(
b_L-b_R,
\sigma_S^2v_{\tau^\star}
\right).
\]

Since
\(S(\bm X)\geq|C_{\tau^\star}(\bm X)|\), the event
\(S(\bm X)\leq\lambda\) implies
\(|C_{\tau^\star}(\bm X)|\leq\lambda\).

If \(b_L-b_R>0\), then
\(|C_{\tau^\star}(\bm X)|\leq\lambda\) implies that the centred
Gaussian error is at most
\(-\{|b_L-b_R|-\lambda\}\).
If \(b_L-b_R<0\), the analogous event is an upper Gaussian tail.
Thus, in either case,
\[
P\left\{
S(\bm X)\leq\lambda
\mid
Y=1,\tau^\star,b_L,b_R
\right\}
\leq
\exp\left\{
-
\frac{
\left(
|b_L-b_R|-\lambda
\right)_+^2
}{
2\sigma_S^2v_{\tau^\star}
}
\right\}.
\]

Using
\(|b_L-b_R|\geq\kappa_S\) and
\(v_{\tau^\star}\leq v_T^{\max}\), and then averaging over any
random segment parameters, yields
\[
P\left\{
S(\bm X)\leq\lambda
\mid Y=1
\right\}
\leq
\exp\left\{
-
\frac{
\left(
\kappa_S-\lambda
\right)_+^2
}{
2\sigma_S^2v_T^{\max}
}
\right\}.
\]

Taking
\(\lambda=\lambda_{S,\alpha_S}\)
and combining the conditional errors with weights
\(\pi_0\) and \(\pi_1\) proves
the bounds stated in Theorem 6.
\end{proof}

\subsection{Proof of Theorem 7 (realised volatility classifiers)}

\begin{proof}
Set \(Z_t=r_t^2\).
By assumption, conditional on either class,
\(Z_1,\ldots,Z_n\) are independent and satisfy
\(0\leq Z_t\leq U^2\).

Under \(Y=0\), we have
\(E\{V(\bm X)\mid Y=0\}\leq n\nu_0\).
Thus, if
\(V(\bm X)>\lambda_{V,\alpha_V}\), then
\[
V(\bm X)
-
E\{V(\bm X)\mid Y=0\}
>
a_{V,\alpha_V}.
\]

Hoeffding's inequality gives
\[
\begin{aligned}
P\left\{
V(\bm X)>\lambda_{V,\alpha_V}
\mid Y=0
\right\}
&\leq
\exp\left\{
-\frac{2a_{V,\alpha_V}^2}{nU^4}
\right\}
\\
&=
\alpha_V.
\end{aligned}
\]

Under \(Y=1\), we have
\(E\{V(\bm X)\mid Y=1\}\geq n(\nu_0+\kappa_V)\).
Therefore,
\[
E\{V(\bm X)\mid Y=1\}
-
\lambda_{V,\alpha_V}
\geq
n\kappa_V-a_{V,\alpha_V}.
\]

Applying the form for the lower tail of Hoeffding's inequality yields
\[
P\left\{
V(\bm X)\leq\lambda_{V,\alpha_V}
\mid Y=1
\right\}
\leq
\exp\left\{
-
\frac{
2\left(
n\kappa_V-a_{V,\alpha_V}
\right)_+^2
}{
nU^4
}
\right\}.
\]

Combining the two bounds conditional on class proves the assertion for the exact rule.

It remains to account for clipping and CNN approximation. On \(E_M\),
\(\operatorname{clip}_M(\bm X)=\bm X\). If
\(\overline h_V(\bm X)=1\) and \(E_M\) occurs, then
\(f_{\epsilon_V,M}^{RV}(\bm X)>
\lambda_{V,\alpha_V}+\epsilon_V\), and uniform approximation implies
\(V(\bm X)>\lambda_{V,\alpha_V}\). Therefore,
\[
\mathbb P\{\overline h_V(\bm X)=1\mid Y=0\}
\leq\alpha_V+\tau_0(M).
\]

If \(\overline h_V(\bm X)=0\) and \(E_M\) occurs, then
\(f_{\epsilon_V,M}^{RV}(\bm X)\leq
\lambda_{V,\alpha_V}+\epsilon_V\), and hence
\(V(\bm X)\leq\lambda_{V,\alpha_V}+2\epsilon_V\). Hoeffding's
inequality therefore gives
\[
\mathbb P\{\overline h_V(\bm X)=0\mid Y=1\}
\leq
\exp\left\{
-\frac{2(n\kappa_V-a_{V,\alpha_V}-2\epsilon_V)_+^2}{nU^4}
\right\}
+\tau_1(M).
\]
Combining the two bounds conditional on class proves the risk bound for the clipped neural classifier in Theorem 7.
\end{proof}

\subsection{Proof of Theorem 8 (autoregressive classifiers)}

\begin{proof}
Fix \(\phi\in\mathbb R\), and let \(\mathbb P_\phi\) and
\(\mathbb E_\phi\) denote probability and expectation conditional on
\(\Phi=\phi\). Set \(\mathcal F_t=\sigma(\varepsilon_1,\ldots,\varepsilon_t)\).
For \(s=0,\ldots,T-1\), define
\[
M_s=\sum_{t=1}^sX_t\varepsilon_{t+1},\qquad
Q_s=\sum_{t=1}^sX_t^2,
\]
with empty sums equal to zero. For every \(\theta>0\) and
\(\xi\in\{-1,1\}\), the process
\[
L_s^{\xi}(\theta)
=\exp\left\{\xi\theta M_s-
\frac{\theta^2\sigma_{AR}^2Q_s}{2}\right\}
\]
is a martingale with mean one with respect to the shifted filtration
\((\mathcal F_{s+1})_{s=0}^{T-1}\). Indeed, \(X_s\) is
\(\mathcal F_s\)-measurable and \(\varepsilon_{s+1}\) is independent of
\(\mathcal F_s\) and Gaussian, so
\(\mathbb E_\phi\{L_s^{\xi}(\theta)\mid\mathcal F_s\}
=L_{s-1}^{\xi}(\theta)\).

Write \(M_T=M_{T-1}\) and \(Q_T=Q_{T-1}\), consistently with the notation
in the main paper. Since
\(X_{t+1}-X_t=(\phi-1)X_t+\varepsilon_{t+1}\),
$S_{AR}(\bm X)=(\phi-1)Q_T+M_T.$ If \(|\phi|\leq1\), then
\(S_{AR}(\bm X)\leq M_T\). Moreover,
\(\operatorname{Var}_\phi(X_t)\leq\sigma_{AR}^2T\), so
\[
\mathbb P_\phi\left(\max_{1\leq t\leq T-1}|X_t|>a\right)
\leq
2T\exp\!\left(-\frac{a^2}{2\sigma_{AR}^2T}\right).
\]
With \(a^2=2\sigma_{AR}^2T\log(4T/\alpha)\), this probability is at most
\(\alpha/2\), and on its complement $Q_T\leq2\sigma_{AR}^2T^2\log(4T/\alpha)=:r_T.$ Using the positive exponential martingale and Markov's inequality, for every \(\theta>0\),
\[
\mathbb P_\phi(M_T>\lambda_{AR,\alpha},Q_T\leq r_T)
\leq
\exp\!\left\{
-\theta\lambda_{AR,\alpha}
+\frac{\theta^2\sigma_{AR}^2r_T}{2}
\right\}.
\]
Taking \(\theta=\lambda_{AR,\alpha}/(\sigma_{AR}^2r_T)\) makes this at
most \(\alpha/2\). Adding the probability of \(Q_T>r_T\) proves the size
bound uniformly over \(|\phi|\leq1\). Integrating this uniform bound with
respect to the conditional law of \(\Phi\) given \(Y=0\) gives the stated
error bound under the null class.

If \(\phi\geq1+\kappa_{AR}\), then on \(Q_T\geq q_T\), the condition
\(q_T\geq2\lambda_{AR,\alpha}/\kappa_{AR}\) implies that
\(S_{AR}(\bm X)\leq\lambda_{AR,\alpha}\) only if \(M_T\leq-\kappa_{AR}Q_T/2\). Using the negative exponential martingale and taking
\(\theta=\kappa_{AR}/(2\sigma_{AR}^2)\),
\[
\mathbb P_\phi\left\{
M_T\leq-\frac{\kappa_{AR}}{2}Q_T,\ Q_T\geq q_T
\right\}
\leq
\exp\!\left(-\frac{\kappa_{AR}^2q_T}{8\sigma_{AR}^2}\right).
\]
This bound is uniform over \(\phi\geq1+\kappa_{AR}\). Integrating it with
respect to the conditional law of \(\Phi\) given \(Y=1\), and then adding
\(\mathbb P(Q_T<q_T\mid Y=1)\leq\eta_T\), proves the assertion for the exact rule.

We now consider the clipped neural classifier. On \(E_M\),
\(\operatorname{clip}_M(\bm X)=\bm X\). Uniform approximation gives
\[
\{\overline h_{AR}(\bm X)=1\}\cap E_M
\subseteq
\{S_{AR}(\bm X)>\lambda_{AR,\alpha}\}.
\]
Thus \(\mathbb P\{\overline h_{AR}(\bm X)=1\mid Y=0\}
\leq\alpha+\tau_0(M)\).

Likewise,
\[
\{\overline h_{AR}(\bm X)=0\}\cap E_M
\subseteq
\{S_{AR}(\bm X)\leq\lambda_{AR,\alpha}+2\epsilon_{AR}\}.
\]
On \(Q_T\geq q_T\), the condition
\(q_T\geq2(\lambda_{AR,\alpha}+2\epsilon_{AR})/\kappa_{AR}\)
implies that the last event can occur only if
\(M_T\leq-\kappa_{AR}Q_T/2\). The same negative
exponential martingale argument as above gives
\[
\mathbb P\{\overline h_{AR}(\bm X)=0\mid Y=1\}
\leq
\eta_T+
\exp\left\{-\frac{\kappa_{AR}^2q_T}{8\sigma_{AR}^2}\right\}
+\tau_1(M).
\]
Combining the two conditional errors proves the bound
\(\mathfrak B_{AR}(\alpha,\epsilon_{AR},M)\) in Theorem 8.
\end{proof}

\subsection{Proof of Theorem 9 (common class oracle bound)}
\begin{proof}
Construct a common CNN architecture by taking the union of the branches
used in Theorems 3, 4 and 5. The realised volatility and autoregressive
branches each receive their own fixed clipping preprocessing. Specifically,
for each \(G\in\{V,AR\}\), let \(M_G\) be the clipping radius fixed in its
comparator and prepend to that branch a ReLU layer with kernel size one and two channels
with
\[
U_t^{G,+}=\sigma(X_t+M_G),\qquad U_t^{G,-}=\sigma(X_t-M_G).
\]
The clipped input is the affine combination
\(\operatorname{clip}_{M_G}(X_t)=-M_G+U_t^{G,+}-U_t^{G,-}\). In the next convolutional
layer, substitute this affine combination for every occurrence of \(X_t\).
The coefficients of \(U_t^{G,+}\) and \(U_t^{G,-}\) are absorbed into that layer's
filter weights, and the constant \(-M_G\) is absorbed into its bias. Its
preactivation therefore coincides exactly with the preactivation of the first hidden layer in the realised volatility or autoregressive construction
evaluated at the clipped input. The remaining layers are unchanged. Thus the
realised volatility and autoregressive branches each gain one convolutional
block. This preprocessing is specific to each branch. The slope change branches
continue to receive the unclipped input.

The first group consists of the
\(2(T-4)\)
branches spanning the full window used to represent the slope change classifier.
The second group consists of the branch with local convolution and global summation
used to construct
\(f_{\epsilon_V,M}^{RV}\),
and the third group consists of the corresponding branch used to
construct
\(f_{\epsilon_{AR},M}^{AR}\).

All branch parameters and all weights of the fixed head are free. The affine
submodel of that head is sufficient for each construction.
To obtain any one of
\(\overline h_S\),
\(\overline h_V\),
or
\(\overline h_{AR}\),
assign the parameters of its corresponding branches as in the relevant
construction and set the output coefficients of all remaining branches
equal to zero.
It follows that all three comparator classifiers belong to the common
class \(\mathcal H\).

Define the empirical risk by
\(\widehat{\operatorname{err}}_N(h)
=
N^{-1}\sum_{i=1}^{N}
\mathbb I\{h(\bm X^{(i)})\neq Y^{(i)}\}\), and set
\[
r_N(d,\delta)
=
\sqrt{\frac{8d\log(2eN/d)+8\log(4/\delta)}{N}}.
\]
By the standard VC uniform deviation inequality, with probability at
least \(1-\delta\),
\[
\sup_{h\in\mathcal H}
\left|
\operatorname{err}(h)
-
\widehat{\operatorname{err}}_N(h)
\right|
\leq
r_N(d_{\mathcal H},\delta).
\]

On this event, the empirical minimality of
\(\widehat h_{\mathrm{ERM}}\)
implies
\[
\operatorname{err}
\left(
\widehat h_{\mathrm{ERM}}
\right)
\leq
\inf_{h\in\mathcal H}
\operatorname{err}(h)
+
2r_N(d_{\mathcal H},\delta).
\]

Since each of the three comparators lies in \(\mathcal H\),
\[
\inf_{h\in\mathcal H}
\operatorname{err}(h)
\leq
\min_{G\in\{S,V,AR\}}
\operatorname{err}(\overline h_G).
\]
Substituting this inequality into the preceding ERM bound proves Theorem 9.
\end{proof}


\section{ Representation and approximation for other architectures}

In the main paper we focus on a general CNN architecture that can represent or approximate several statistics related to bubbles. In this section, we provide additional representation results for other neural network architectures. In particular, we consider fully connected ReLU networks, block sparse fully connected networks and ResNet CNNs, and show how they can represent or approximate the range statistic and the AR explosiveness statistic. The block sparse and ResNet constructions are particularly useful for understanding how the AR explosiveness statistic can be decomposed into local interaction terms and implemented by alternative neural network architectures.

We first show that the range statistic can be exactly represented by a fully connected ReLU neural network with one hidden layer.

\begin{theorem}
    Let \(T\geq2\). Define the range rule statistic by $R(X)=\max_i X_i-\min_i X_i$ for $\bm{X}=(X_1, \cdots, X_T) \in \mathbb{R}^T$. For any threshold $\lambda>0$, define the range rule classifier $h_{range,\lambda}(X)=\mathbb{I}\{ R(X)>\lambda \}$, then
    \begin{equation}
       h_{range,\lambda}(X) \in \mathcal{H}_{1,T(T-1)},
    \end{equation}
    where $\mathcal{H}_{1,m}$ denotes the thresholded classifier class obtained by applying \(\mathbb I\{\cdot>0\}\) to the real output of a fully connected ReLU network with one hidden layer and \(m\) hidden nodes.
    \label{thm:range_fnn}
\end{theorem}

\begin{proof}
For $1\leq i,j\leq T$ with $i\neq j$, let $v_{ij}=e_i-e_j\in\mathbb R^T$, where $e_i$ denotes the $i$th canonical basis vector in $\mathbb R^T$. Then,
\begin{equation*}
R(X)=\max_{1\leq i\leq T}X_i-\min_{1\leq j\leq T}X_j = \max_{i\neq j}(X_i-X_j) = \max_{i\neq j}v_{ij}^{\top}X.
\end{equation*}
Hence, $h_{\mathrm{range},\lambda}(X) = \mathbb I\{R(X)>\lambda\} = \mathbb I \{ \max_{i\neq j}(v_{ij}^{\top}X-\lambda)>0 \}$. Equivalently, $h_{\mathrm{range},\lambda}(X) = \mathbb I \{ \sum_{i\neq j}\sigma(v_{ij}^{\top}X-\lambda)>0 \}$, where $\sigma(x)=\max\{x,0\}$ is the ReLU activation function.

We now construct a ReLU network with one hidden layer. Let $\bm{W}_0$ be the $T(T-1)\times T$ matrix whose rows are given by $v_{ij}^{\top}$ for all ordered pairs $(i,j)$ with $i\neq j$. Let $b_1=\lambda\mathbf 1_{T(T-1)}$ and let $\bm{W}_1=\mathbf 1_{T(T-1)}^{\top}$. Then $\bm{W}_1\sigma(\bm{W}_0\bm{X}-b_1) = \sum_{i\neq j}\sigma(v_{ij}^{\top}\bm{X}-\lambda)$. Applying the final threshold activation $\sigma_0^*(z)=\mathbb I\{z>0\}$ gives
\begin{equation*}
\sigma_0^* (W_1\sigma(W_0X-b_1)) = \mathbb I \{ \sum_{i\neq j}\sigma(v_{ij}^{\top}X-\lambda)>0 \} = h_{\mathrm{range},\lambda}(X).
\end{equation*}
Therefore, $h_{\mathrm{range},\lambda}\in\mathcal H_{1,T(T-1)}$. This completes the proof.
\end{proof}

Theorem 5 of the main paper shows that the AR statistic can be approximated by the general CNN. We now give alternative block sparse FNN and ResNet constructions, following \citet{oono2019approximation}.

A block sparse fully connected network consists of several fully connected blocks arranged in parallel, followed by a single affine output layer. ResNet is a neural network architecture consisting of a sequence of residual blocks, where each block adds a learned nonlinear transformation to its input through a skip connection. The architectures and precise definitions, as well as schematic views of these two architectures, can be found in Sections 3.3 and 3.2 of \cite{oono2019approximation}, respectively.

We first show that the AR explosiveness statistic can be approximated by a block sparse FNN. Its local multiplication terms can be approximated independently and then summed.

\begin{theorem}
    Let \(T\geq2\). For any \(M>0\) and \(\epsilon>0\), there exists a block sparse fully connected neural network \(f_{\epsilon,M}^{FNN}\) with \(T-1\) blocks, constant width \(C\), and block depth \(L=O\{\log(2+(T-1)M^2/\epsilon)\}\) such that
    \begin{equation}
        \sup_{\bm X\in [-M,M]^T}|f_{\epsilon,M}^{FNN}(\bm X)-S_{AR}(\bm X)| \le \epsilon,
    \end{equation}
    where \(f_{\epsilon,M}^{FNN}=\sum_{t=1}^{T-1}f_{t,\epsilon,M}\), and the \(t\)-th block approximates \(g_t(\bm X)=X_t(X_{t+1}-X_t)\).
    \label{thm:ar_fnn}
\end{theorem}

\begin{proof}
    Normalise the domain by setting \(U_j=X_j/M\), \(j=1,\ldots,T\). Then \(\bm U\in[-1,1]^T\) and
    \begin{equation*}
        S_{AR}(\bm X)=M^2\sum_{t=1}^{T-1}U_t(U_{t+1}-U_t)=:M^2S_{AR}^0(\bm U).
    \end{equation*}
    For each \(t=1,\ldots,T-1\), write \(g_t^0(\bm U)=U_tU_{t+1}-U_t^2\), and set \(\delta=\min\{1/2,\epsilon/[2(T-1)M^2]\}\). By the standard ReLU multiplication approximation \citep{yarotsky2017error}, there exists a ReLU network of constant width \(\operatorname{Mult}_{\delta}\) such that \(\lvert\operatorname{Mult}_{\delta}(a,b)-ab\rvert\leq\delta\) on \([-1,1]^2\), with depth \(O\{\log(1/\delta)\}\).

    Define \(\widetilde g_{t,\delta}(\bm U)=\operatorname{Mult}_{\delta}(U_t,U_{t+1})-\operatorname{Mult}_{\delta}(U_t,U_t)\). Within each $t$-block, the two multiplication networks run in parallel and a final affine output takes their difference. Their combined width is therefore at most twice the constant width of one multiplication network and remains constant. Uniformly over \(\bm U\in[-1,1]^T\),
    \begin{equation*}
        \begin{aligned}
        | \widetilde g_{t,\delta}(\bm U)-g_t^0(\bm U)| &\leq
        | \operatorname{Mult}_{\delta}(U_t,U_{t+1})-U_tU_{t+1}|\\
        &\quad+| \operatorname{Mult}_{\delta}(U_t,U_t)-U_t^2| \le 2\delta.
        \end{aligned}
    \end{equation*}
    Thus, with \(\widetilde S_{AR,\delta}^{0}=\sum_{t=1}^{T-1}\widetilde g_{t,\delta}\),
    \begin{equation*}
        \sup_{\bm U\in[-1,1]^T}|\widetilde S_{AR,\delta}^{0}(\bm U)-S_{AR}^{0}(\bm U)| \le 2(T-1)\delta.
    \end{equation*}
    Define \(f_{\epsilon,M}^{FNN}(\bm X)=M^2\widetilde S_{AR,\delta}^{0}(\bm X/M)\). Then
    \begin{equation*}
        \sup_{\bm X\in[-M,M]^T}|f_{\epsilon,M}^{FNN}(\bm X)-S_{AR}(\bm X)|
        \leq M^2\,2(T-1)\delta\leq\epsilon.
    \end{equation*}

    The \(t\)-th of the \(T-1\) blocks uses only \((X_t,X_{t+1})\). Since each multiplication network has constant width and depth \(O\{\log(1/\delta)\}=O\{\log(2+(T-1)M^2/\epsilon)\}\), the claimed architecture follows.
\end{proof}

The same AR explosiveness statistic can also be approximated by a ResNet CNN.

\begin{theorem}
    Let \(T\geq2\), set \(D=T\), and fix \(K\in\{2,\ldots,T\}\). For any \(M>0\) and \(\epsilon>0\), there exists a ResNet CNN \(f_{\epsilon,M}^{CNN}\) with \(T-1\) residual blocks, block depth \(l_{CNN}\leq L+\lceil(T-1)/(K-1)\rceil\), channel size \(C_{CNN}\leq4C\), and filter size at most \(K\), such that
    \[
    \sup_{\bm X\in[-M,M]^T}|f_{\epsilon,M}^{CNN}(\bm X)-S_{AR}(\bm X)|\leq\epsilon,
    \]
    where \(L\) and \(C\) are the block depth and width in Theorem S.2.
    \label{thm: ar_resnet}
\end{theorem}

\begin{proof}
    The block sparse construction in Theorem S.2 has input dimension \(D=T\), \(T-1\) blocks, block depth \(L\), and width \(C\). Its normalisation by \(M\) and output rescaling by \(M^2\) are affine maps, so they can be absorbed into the first and final layers without changing the block count or approximation error. Theorem 1 of \citet{oono2019approximation} converts such a block sparse ReLU network into a ResNet CNN with the same number of residual blocks, block depth at most \(L+\lceil(D-1)/(K-1)\rceil\), channel size at most \(4C\), and filter size at most \(K\). Substituting \(D=T\) gives the stated architecture, and the conversion preserves the computed function, so the error remains at most \(\epsilon\).
\end{proof}

The explicit architecture also permits an ERM statement for each architecture without repeating a separate risk derivation for each model. For \(\mathcal A\in\{\mathrm{FNN},\mathrm{CNN}\}\), let \(f_{\epsilon_{AR},M}^{\mathcal A}\) denote the approximation in Theorem S.2 or S.3 and define
\[
\overline h_{AR}^{\mathcal A}(\bm X)
=\mathbb I\left\{
f_{\epsilon_{AR},M}^{\mathcal A}(\operatorname{clip}_M(\bm X))
>\lambda_{AR,\alpha}+\epsilon_{AR}
\right\}.
\]


\begin{cor}[ERM bound for each architecture]
Let \(\mathcal A\in\{\mathrm{FNN},\mathrm{CNN}\}\), and let \(\mathcal H_A\) be a fixed block sparse FNN or ResNet classifier class containing the corresponding clipped neural AR comparator \(\overline h_{AR}^{\mathcal A}\). Let \(d_A=\operatorname{VCdim}(\mathcal H_A)\), with \(1\leq d_A\leq N\), and let
\[
\widehat h_A\in\arg\min_{h\in\mathcal H_A}
\frac{1}N\sum_{i=1}^N\mathbb I\{h(\bm X^{(i)})\neq Y^{(i)}\}.
\]
Under the assumptions of Theorem 8, for every \(\delta\in(0,1)\), with probability at least \(1-\delta\),
\[
\operatorname{err}(\widehat h_A)
\leq
\mathfrak B_{AR}(\alpha,\epsilon_{AR},M)
+2\sqrt{\frac{8d_A\log(2eN/d_A)+8\log(4/\delta)}{N}}.
\]
\end{cor}

\begin{proof}
The standard VC inequality controlling deviations on both sides and empirical minimality give
\[
\operatorname{err}(\widehat h_A)
\leq
\inf_{h\in\mathcal H_A}\operatorname{err}(h)
+2\sqrt{\frac{8d_A\log(2eN/d_A)+8\log(4/\delta)}{N}}.
\]
Since \(\overline h_{AR}^{\mathcal A}\in\mathcal H_A\), the infimum is at most
\(\operatorname{err}(\overline h_{AR}^{\mathcal A})\), which is bounded by
\(\mathfrak B_{AR}(\alpha,\epsilon_{AR},M)\) in Theorem 8.
\end{proof}

\section{ Additional Simulation}
For reproducibility of the oracle comparison in Section~4 of the main paper, the affine logistic head was fitted using inverse regularisation strength \(C=10^4\), the \texttt{lbfgs} solver and at most \(2{,}000\) iterations. The global random seed was 40, and the calibration, scenario, replication and nested-sample seeds were derived deterministically from this value. The logistic routine also used \texttt{random\_state=40}.
\subsection{Numerical Verification of Representation}
\label{sec:representation}
This subsection numerically verifies the representation and approximation results established in Section 3. In particular, we examine whether the constructive CNNs reproduce the range, maximum drawup, maximum drawdown, and slope change rules up to numerical precision. All network parameters in this experiment are assigned directly according to the theoretical constructions. No model training or parameter estimation is involved.

All input windows are rescaled to \(\mathcal X_M\), with \(M=1\). We consider three representative window lengths \(T\in\{20,40,80\}\). For each \(T\), we generate \(N=1000\) evaluation windows by simulating structured time series trajectories with Gaussian innovations, covering both cases without and with change points.

Each trajectory is initially generated from $Z_1=\varepsilon_1, Z_t=\phi Z_{t-1}+\varepsilon_t$, $t=2,\ldots,T$, where $\phi\sim{U}(0.3,0.98), s\sim{U}(0.04,0.18), \varepsilon_t\overset{\mathrm{iid}}{\sim}{N}(0,s^2)$. The first 500 trajectories contain no change point. Each of the remaining 500 trajectories contains one change point. Its location $\tau$ is sampled uniformly from $\{ \max\{3,\lfloor T/4\rfloor\}, \ldots, \min\{T-3,\lfloor 3T/4\rfloor\}-1 \}$. The trajectories are assigned cyclically to three perturbation mechanisms. For a trend change, a linear sequence increasing from zero to $A\sim{U}(-1,1)$ is added after $\tau$. For a volatility change, independent ${N}(0,(2.5s)^2)$ noise is added after $\tau$. For a level shift, a common displacement $A\sim{U}(-0.9,0.9)$ is added after $\tau$. Each resulting trajectory $\widetilde Z$ is rescaled as
\begin{equation*}
        X_t = 0.95M \frac{\widetilde Z_t}{\max_{1\leq j\leq T}|\widetilde Z_j|}.
\end{equation*}
The generator of structured windows uses seed \(T+4\). Consequently, the experiment evaluates a total of \(3\times1000=3000\) bounded windows.
The numerical check follows the level of the corresponding theorem. Theorem~1 asserts equality of scores, so for the range we report the maximum and mean absolute differences between the network output plus \(\lambda\) and the directly computed statistic. Writing $e_i=|f_{R,\lambda}(X_i)+\lambda-\mathcal R(X_i)|$, these quantities are $E_{\max}=\max_{1\leq i\leq N}e_i$ and $E_{\mathrm{mean}}=N^{-1}\sum_{i=1}^N e_i$. Theorems~2 and~3 assert equality of classifiers rather than statistics. For maximum drawup, maximum drawdown and slope change, we therefore set \(\lambda\) to the empirical median of the statistic and report the percentage of windows on which the constructed CNN classifier agrees with the direct rule.

Table~\ref{tab:exact_cnn_representation} shows range errors at the precision of floating point arithmetic and complete classifier agreement for the other three rules, as predicted by the representation results.

\begin{table}[htbp]
\centering
\caption{Numerical check of the exact representations. For the range, \(E_{\max}\) and \(E_{\mathrm{mean}}\) are reported in units of \(10^{-7}\) and \(10^{-8}\), respectively. The remaining entries are classifier agreement percentages at thresholds equal to the empirical medians of the corresponding statistics.}
\label{tab:exact_cnn_representation}
\setlength{\tabcolsep}{15pt}
{\small
\begin{tabular}{c cc ccc}
\toprule
& \multicolumn{2}{c}{Range}
& \multicolumn{3}{c}{Classifier agreement (\%)} \\
\cmidrule(lr){2-3}
\cmidrule(lr){4-6}
$T$
& $E_{\max}$ & $E_{\mathrm{mean}}$
& Max drawup & Max drawdown & Slope change \\
\midrule
20 & 2.384 & 4.929 & 100.00 & 100.00 & 100.00 \\
40 & 2.384 & 4.417 & 100.00 & 100.00 & 100.00 \\
80 & 2.384 & 4.417 & 100.00 & 100.00 & 100.00 \\
\bottomrule
\end{tabular}
}
\end{table}

\subsection{Sampled Numerical Consistency Check for Approximation}
We next conduct a sampled numerical check of the CNN approximations to realised volatility and the AR statistic against the analytic uniform bounds in Theorems 4 and 5. The uniform statements follow from the theorems rather than from this finite collection of trajectories. The square function is approximated using the iterated ReLU construction based on a triangular map described in the proofs. We consider approximation levels $m=\{1,2,3,4,5,6\}$.

For the realised volatility statistic, the resulting uniform error bound is $\varepsilon_{\mathrm{RV}}(T,M,m) =(T-1)(2M)^2 4^{-(m+1)}$. For the AR statistic, the corresponding bound is $\varepsilon_{\mathrm{AR}}(T,M,m) = 3(T-1)M^2 4^{-(m+1)}$. For each statistic $G$, window length $T$, and approximation level $m$, we compute $E_{\max}(G,T,m) = \max_{1\leq i\leq N}|\widehat G_m(X_i)-G(X_i)|$ and $E_{\mathrm{mean}}(G,T,m)=N^{-1}\sum_{i=1}^N|\widehat G_m(X_i)-G(X_i)|$. We check on the sampled trajectories that $E_{\max}(G,T,m) \leq  \varepsilon_G(T,M,m)+10^{-6}$, where \(10^{-6}\) is included only as a numerical tolerance.

To examine whether statistic approximation preserves the associated classification rule, we use the empirical median of the exact statistic as a diagnostic threshold, $\lambda_{G,T} = \operatorname{median}\{ G(X_i):i=1,\ldots,N\}$. The labels from the exact statistic and the CNN are defined as $Y_i = \mathbb{I} \{ G(X_i)>\lambda_{G,T}\}$, $\widehat Y_i = \mathbb{I}\{\widehat G_m(X_i)>\lambda_{G,T}\}$. We report their empirical agreement $A(G,T,m) = {1}/{N} \sum_{i=1}^{N} \mathbb{I}\{Y_i=\widehat Y_i\}$.

Table~\ref{tab:approximation_results} presents the approximation results for realised volatility and the AR statistic. For both statistics, the maximum and mean absolute errors decrease rapidly as the ReLU approximation level $m$ increases. This reduction is accompanied by a corresponding increase in classification agreement. For a fixed approximation level, longer windows generally produce larger approximation errors and lower agreement rates. This behaviour is reasonable
because both statistics aggregate local quantities over the entire window, so a longer sequence contains more local approximation errors to be accumulated.

\begin{table}[htbp]
\centering
\caption{Approximation errors and classification agreement for the realised volatility and AR statistics. Here, agreement is the percentage of trajectories for which the approximate and
exact statistics produce the same threshold decision.}
\label{tab:approximation_results}

\setlength{\tabcolsep}{6.5pt}
\renewcommand{\arraystretch}{1.05}

{\small
\begin{tabular}{cc ccc ccc}
\toprule
\multirow{2}{*}{$T$}
& \multirow{2}{*}{Level $m$}
& \multicolumn{3}{c}{Realised volatility}
& \multicolumn{3}{c}{AR statistic} \\
\cmidrule(lr){3-5}
\cmidrule(lr){6-8}
&
& $E_{\max}$ & $E_{\mathrm{mean}}$ & Agreement (\%)
& $E_{\max}$ & $E_{\mathrm{mean}}$ & Agreement (\%) \\
\midrule

\multirow{6}{*}{20}
& 1 & 4.021 & 2.701 & 59.10
    & 1.225 & 0.665 & 78.70 \\
& 2 & 1.011 & 0.766 & 85.50
    & 0.305 & 0.191 & 93.80 \\
& 3 & 0.253 & 0.197 & 96.90
    & 0.085 & 0.050 & 99.00 \\
& 4 & 0.066 & 0.049 & 99.50
    & 0.019 & 0.012 & 99.60 \\
& 5 & 0.015 & 0.012 & 99.90
    & 0.005 & 0.003 & 99.80 \\
& 6 & 0.004 & 0.003 & 100.00
    & 0.001 & 0.001 & 100.00 \\
\midrule

\multirow{6}{*}{40}
& 1 & 7.632 & 5.230 & 57.40
    & 2.223 & 1.298 & 74.90 \\
& 2 & 1.955 & 1.540 & 81.10
    & 0.571 & 0.383 & 91.10 \\
& 3 & 0.492 & 0.404 & 95.00
    & 0.146 & 0.102 & 98.30 \\
& 4 & 0.123 & 0.102 & 98.30
    & 0.037 & 0.025 & 99.90 \\
& 5 & 0.030 & 0.025 & 99.30
    & 0.009 & 0.006 & 99.90 \\
& 6 & 0.008 & 0.006 & 99.90
    & 0.002 & 0.002 & 99.90 \\
\midrule

\multirow{6}{*}{80}
& 1 & 14.090 & 10.170 & 55.10
    & 3.862 & 2.553 & 70.70 \\
& 2 & 3.782 & 3.090 & 77.60
    & 1.053 & 0.777 & 88.80 \\
& 3 & 0.943 & 0.819 & 94.60
    & 0.266 & 0.206 & 97.30 \\
& 4 & 0.235 & 0.206 & 98.60
    & 0.067 & 0.051 & 99.50 \\
& 5 & 0.059 & 0.051 & 99.70
    & 0.016 & 0.013 & 99.80 \\
& 6 & 0.015 & 0.013 & 99.90
    & 0.004 & 0.003 & 99.90 \\
\bottomrule
\end{tabular}
}
\end{table}

For comparison, we instantiate the approximation networks using the iterated ReLU construction based on a triangular map. Let $Q_{m,A}$ denote the resulting approximation at level $m$ to the square function on $[-A,A]$. This construction satisfies $\sup_{|z|\leq A} \left|Q_{m,A}(z)-z^2\right| \leq A^2 4^{-(m+1)}$ by \cite{yarotsky2017error}. The numerical AR approximant used in the experiment is
\[
\widehat S_{AR,m}(\bm X)=\sum_{t=1}^{T-1}\left[
\frac{Q_{m,2M}(X_t+X_{t+1})-Q_{m,2M}(X_t-X_{t+1})}{4}
-Q_{m,M}(X_t)\right].
\]
For realised volatility, each return $X_{t+1}-X_t$ belongs to $[-2M,2M]$. Summing the local errors in the square approximation over the $T-1$ returns therefore gives $\epsilon_{\mathrm{RV}}(T,M,m) = (T-1)(2M)^2 4^{-(m+1)}$. For the AR statistic, multiplication is implemented through the polarization identity $ab=\{(a+b)^2-(a-b)^2\}/{4}$. The approximation error for the product $ab$ is consequently bounded by $2M^2 4^{-(m+1)}$, while the additional approximation of $a^2$ contributes $M^2 4^{-(m+1)}$. Summing over the $T-1$ local AR terms yields $\epsilon_{\mathrm{AR}}(T,M,m) = 3(T-1)M^2 4^{-(m+1)}$.

Figure~\ref{fig:approximation_error_depth} further illustrates the geometric decay of the approximation error. All sampled maximum errors remain below their corresponding analytic upper bounds. The separation is particularly visible for the AR statistic, indicating that the analytic bound is conservative for the simulated trajectories considered here.

\begin{figure}[htpb]
    \centering
    \includegraphics[width=1\linewidth]{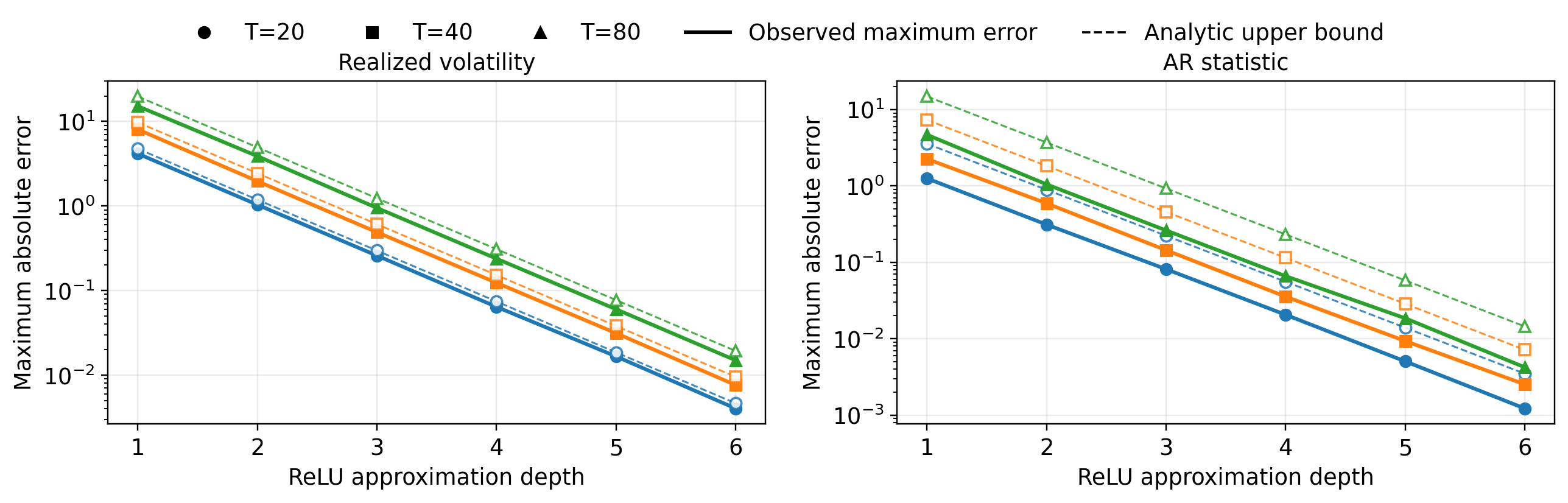}
    \caption{Maximum absolute approximation error as a function of ReLU approximation level $m$ for realised volatility (left) and the AR statistic (right). The vertical axes are displayed on a logarithmic scale.}
    \label{fig:approximation_error_depth}
\end{figure}

\end{document}